\documentclass[10pt]{article}

\usepackage[utf8]{inputenc}
\usepackage[T1]{fontenc}

\usepackage{epsf}
\usepackage{amsmath}

\allowdisplaybreaks

\usepackage[showframe=false]{geometry}
\usepackage{changepage}

\usepackage{epsfig}
\usepackage{amssymb}

\usepackage{amsthm}
\usepackage{setspace}
\usepackage{cite}
\usepackage{mcite}

\usepackage{algorithmic}  
\usepackage{algorithm}

\usepackage{shadow}
\usepackage{fancybox}
\usepackage{fancyhdr}

\usepackage{color}
\usepackage[usenames,dvipsnames,svgnames,table]{xcolor}
\newcommand{\bl}[1]{\textcolor{blue}{#1}}

\definecolor{mypurple}{rgb}{.4,.0,.5}

\usepackage[hyphens]{url}

\usepackage[colorlinks=true,
            linkcolor=black,
            urlcolor=blue,
            citecolor=purple]{hyperref}

\usepackage{breakurl}

\def\y{{\bf y}}

\def\x{{\bf x}}

\def\x{{\mathbf x}}

\def\x{{\bf x}}
\def\y{{\bf y}}
\def\z{{\bf z}}

\def\h{{\bf h}}

\def\tr{\mbox{Tr}}

\def\tr{{\rm tr}\,}

\def\cB{{\mathcal B}}

\def\cG{{\mathcal G}}

\def\cI{{\mathcal I}}

\def\be{\begin{equation}}
\def\ee{\end{equation}}
\def\ba{\left[\begin{array}}
\def\ea{\end{array}\right]}

\def\x{{\bf x}}
\def\y{{\bf y}}
\def\z{{\bf z}}

\def\1{{\bf 1}}

\def\g{{\bf g}}
\def\0{{\bf 0}}

\def\mR{{\mathbb R}}

\def\mE{{\mathbb E}}

\def\mP{{\mathbb P}}

\def\lp{\left (}
\def\rp{\right )}

\def\y{{\bf y}}

\def\x{{\bf x}}

\def\x{{\mathbf x}}

\def\x{{\bf x}}
\def\y{{\bf y}}
\def\z{{\bf z}}

\def\h{{\bf h}}

\def\tr{\mbox{Tr}}

\def\tr{{\rm tr}\,}

\def\cB{{\cal B}}

\def\be{\begin{equation}}
\def\ee{\end{equation}}
\def\ba{\left[\begin{array}}
\def\ea{\end{array}\right]}

\def\x{{\bf x}}
\def\y{{\bf y}}
\def\z{{\bf z}}

\def\R{{\bf R}}

\def\({\left (}
\def\){\right )}

\def\1{{\bf 1}}

\def\g{{\bf g}}
\def\0{{\bf 0}}

\usepackage{xcolor}
\usepackage{color}

\definecolor{darkgreen}{rgb}{0, 0.4,0}

\definecolor{purplebrown}{rgb}{0.5,0.1,0.6}

\definecolor{ultclupcol}{rgb}{0.1,0.5,0.5}

\definecolor{mytrycolor}{rgb}{0.5,0.7,0.2}

\definecolor{ultclupcola}{rgb}{.5,0,.5}

\definecolor{shadebrown}{rgb}{0.1,0.1,0.9}
\definecolor{lightblue}{rgb}{0.2,0,1}

\usepackage{fancybox}
\usepackage{graphicx}
\usepackage{epstopdf}
\usepackage{epsfig}
\usepackage{wrapfig}
\usepackage{subfigure}

\usepackage{xcolor}
\usepackage{tcolorbox}

\newtcbox{\xmybox}{on line,
arc=7pt,
before upper={\rule[-3pt]{0pt}{10pt}},boxrule=0pt,
boxsep=0pt,left=6pt,right=6pt,top=0pt,bottom=0pt,enhanced, coltext=blue, colback=white!10!yellow}

\newtcbox{\xmyboxa}{on line,
arc=7pt,
before upper={\rule[-3pt]{0pt}{10pt}},boxrule=0pt,
boxsep=0pt,left=6pt,right=6pt,top=0pt,bottom=0pt,enhanced, colback=white!10!yellow}

\newtcbox{\xmyboxb}{on line,
arc=7pt,
before upper={\rule[-3pt]{0pt}{10pt}},boxrule=1pt,colframe=darkgreen!100!blue,
boxsep=0pt,left=6pt,right=6pt,top=0pt,bottom=0pt,enhanced, colback=white!10!yellow}

\newtcbox{\xmyboxc}{on line,
arc=7pt,
before upper={\rule[-3pt]{0pt}{10pt}},boxrule=.7pt,colframe=blue!100!blue,
boxsep=0pt,left=6pt,right=6pt,top=0pt,bottom=0pt,enhanced, coltext=blue, colback=white!10!yellow}

\newtcbox{\xmytboxa}{on line,
arc=7pt,
before upper={\rule[-3pt]{0pt}{10pt}},boxrule=.0pt,colframe=pink!50!yellow,
boxsep=0pt,left=6pt,right=6pt,top=0pt,bottom=0pt,enhanced, coltext=white, colback=blue!40!red}

\newtcbox{\xmytboxb}{on line,
arc=7pt,
before upper={\rule[-3pt]{0pt}{10pt}},boxrule=.0pt,colframe=pink!50!yellow,
boxsep=0pt,left=6pt,right=6pt,top=0pt,bottom=0pt,enhanced, coltext=white, colback=white!40!green}

\makeatletter
\newcommand\subsubsubsection{\@startsection{paragraph}{4}{\z@}{-2.5ex\@plus -1ex \@minus -.25ex}{1.25ex \@plus .25ex}{\normalfont\normalsize\bfseries}}
\newcommand\subsubsubsubsection{\@startsection{subparagraph}{5}{\z@}{-2.5ex\@plus -1ex \@minus -.25ex}{1.25ex \@plus .25ex}{\normalfont\normalsize\bfseries}}
\makeatother

\newtheorem{theorem}{Theorem}

\newtheorem{lemma}{Lemma}

\begin{document}

\begin{singlespace}

\title{Phase retrieval with rank $d$ measurements -- \emph{descending} algorithms phase transitions  
}
\author{
\textsc{Mihailo Stojnic
\footnote{e-mail: {\tt flatoyer@gmail.com}} }}
\date{}
\maketitle

\centerline{{\bf Abstract}} \vspace*{0.1in}

Companion paper \cite{Stojnicphretreal24} developed a powerful \emph{Random duality theory} (RDT) based analytical program to statistically characterize performance of \emph{descending} phase retrieval algorithms (dPR) (these include all variants of gradient descents and among them widely popular Wirtinger flows). We here generalize the program and show how it can be utilized to handle rank $d$ positive definite phase retrieval (PR) measurements (with special cases $d=1$ and $d=2$ serving as emulations of the real and complex phase retrievals, respectively). In particular, we observe that the minimal  sample complexity ratio (number of measurements scaled by the dimension of the unknown signal) which ensures dPR's success  exhibits a phase transition (PT) phenomenon. For both plain and lifted RDT we determine phase transitions locations. To complement theoretical results we implement a log barrier gradient descent variant and observe that, even in small dimensional scenarios (with problem sizes on  the order of 100), the simulated phase transitions are in an excellent agreement with the theoretical predictions.

\vspace*{0.25in} \noindent {\bf Index Terms: Phase retrieval; Rank $d$ measurements; Descending algorithms; Random duality theory}.

\end{singlespace}

\section{Introduction}
\label{sec:back}

Assume that a unit norm $\bar{\x}\in\R^n$ is measured through the collection of  \emph{phaseless} measurements
\begin{eqnarray}
   \y &=& |A\bar{\x}|^2, \label{eq:inteq1}
\end{eqnarray}
where $A\in\mR^{m\times n}$.  Provided access to $\y$ and $A$, recovering $\bar{\x}$ amounts to solving the following system of quadratic equations
\begin{eqnarray}
\mbox{find} & & \x \nonumber \\
   \mbox{subject to} & & |A\x|^2=\y (= |A\bar{\x}|^2). \label{eq:inteq2}
\end{eqnarray}
The above is the \emph{phase retrieval} (PR) problem in its real form. In a host of signal and image processing protocols acquiring signal's phase is technically infeasible. In such scenarios efficiently solving PRs appears as a key mathematical/algorithmic  challenge that allows practical data acquisition and recovery.

Fairly reach history of PR applications dates back to early days of x-ray crystallography \cite{Harrison93,Millane90,Millane06}. Expansion of  x-ray crystallography to noncrystalline materials, allowed PR to became an irreplaceable  component of coherent diffraction, ptychography, optical, microscopic, and  astronomical  imaging (see, e.g., \cite{Thibault08,Hurt89,KST1995,Miao1999ExtendingTM,ShechtmanECCMS15,MISE08,Bunk07,Walther01011963,Fienup87,Fienup78,Rod08,Dierolf10,BS79,Misell73}).  As understanding of main PR principles improved over the last several decades, the range of applications within and outside image processing widened. Some examples  of further extensions include digital holography \cite{Duadi11,Gabor48,Gabor65}, quantum physics \cite{Corb06,Hein13,HaahHJWY17}, blind deconvolution/demixing \cite{MWCC19,LLSW19,Jung17,ARJ13}, and many others.

\subsection{A first look at PR's difficulty}
\label{sec:probdiff}

In this paper we consider practically relevant scenario where the underlying problem dimensions are large. Moreover, we consider mathematically the most challenging  \emph{linear/proportional} regime
\begin{eqnarray}
 \alpha \triangleq \lim_{n\rightarrow\infty} \frac{m}{n}, \label{eq:inteq3}
\end{eqnarray}
where $\alpha$ is the so-called \emph{oversampling or sample complexity ratio} which remains constant as $n$ and $m$ grow. Two aspects of (\ref{eq:inteq2}) are important. The first one is of predominantly theoretical character and relates to the ``mathematical soundness''. It particularly  focuses on algebraic  characterizations of uniqueness, injectivity, and stability  (see, e.g., \cite{Conca15,Balan06,Balan09,Bande14,Vinz15} and references therein). The second is of more practical relevance and relates to algorithmic solving of (\ref{eq:inteq2}) (since  both $\bar{\x}$ and $-\bar{\x}$ are  admissible solutions,  solving up to a \emph{global} phase is assumed throughout the paper). The key algorithmic difficulty is the appearance of (squared) magnitudes. A naive removal of magnitudes transforms the problem into an easily solvable linear system but incurs an exponential complexity (on the order of $O(2^m)$) and  is practically useless.

\subsection{PR with rank $d$ measurements}
\label{sec:probdiff}

Since the simple methods (along the lines of the above naive combinatorial approach) were not computationally satisfactory, development of computationally more efficient solvers ensued and a rather vast algorithmic theory has been built. Particularly relevant progress has been made over the last 10-15 years and is in large part due to the appearance of breakthrough papers  \cite{CandesSV13,CandesESV13,CandesLS15}. One of these breakthroughs, \cite{CandesLS15}, uncovered that a simple gradient algorithm (also called Wirtinger flow) can often act as a very successful PR solver. Companion paper \cite{Stojnicphretreal24} made a strong theoretical progress in justifying such a remarkable gradient property. In particular, \cite{Stojnicphretreal24} provided \emph{phase transition} level of performance analysis of a more general class of  \emph{descending} phase retrieval  algorithms (dPR). Here we make a substantial progress, move things on a higher level and provide similar phase transition type of analysis for the so-called rank $d$ measurements phase retrieval (basically a rank $d$ variant of (\ref{eq:inteq2})).

As stated above, (\ref{eq:inteq2}) is a phase retrieval real variant. Within the rank $d$ measurements context it is obtained as a special case, $d=1$. For $d=2$ one then obtains the emulation of the complex phase retrieval -- a more relevant scenario in some imaging application. Higher ranks $d$ are providing further theoretical support when the measurements' structure can be controlled beyond what is typically seen in the standard PR imaging applications (for similar higher rank measurements considerations see e.g., \cite{Dux16,Basu00,Zehni19,Zehni20,Bill10,Juhas06,Huang19,Huangetal20,Barahona88}).

While the standard phase retrieval is usually described through (\ref{eq:inteq1}) and (\ref{eq:inteq2}),  (\ref{eq:inteq1}) can also be rewritten in the following (slightly different) form
\begin{eqnarray}
   \bar{\y}_i = |A_{i,:}\bar{\x}|^2 = \bar{\x}^T A_{i,:}^T A_{i,:}\bar{\x} = \bar{\x}^T \cB_{(i)}\bar{\x}, \quad \cB_{(i)}=A_{i,:}^T A_{i,:}, \quad i=1,2,\dots,m, \label{eq:ex0inteq1}
\end{eqnarray}
where $A_{i,:}\in\mR^{1\times n}$ is the $i$-th row of $A$. Moreover, the above matrix form of measurements allows one to further describe  the data acquisition as
\begin{eqnarray}
   \bar{\y}_i &=&  \bar{\x}^T \cB_{(i)}\bar{\x}, \quad   \cB_{(i)}^T=\cB_{(i)}\geq 0 \mbox{ is rank $1$}, \quad i=1,2,\dots,m. \label{eq:ex0inteq2}
\end{eqnarray}
 Formulation from (\ref{eq:ex0inteq2}) implies that (\ref{eq:inteq1}) is  the so-called \emph{real} phase retrieval (PR) with rank 1 positive-semidefinite (PSD) measurements. Its a rank $d$ analogue is then naturally given as
\begin{eqnarray}
\hspace{-.08in}\bl{\textbf{\emph{Rank $d$ measurements (RdM):}}} \quad\quad   \bar{\y}_i &=&  \bar{\x}^T \cB_{(i)}\bar{\x}, \quad   \cB_{(i)}^T=\cB_{(i)}\geq 0 \mbox{ is rank $d$}, \quad i=1,2,\dots,m. \label{eq:ex0inteq3}
\end{eqnarray}
Recovering $\bar{\x}$ provided the access to $\bar{\y}$ and rank $d$ PSD $\cB_{(i)}$ assumes solving the following inverse feasibility problem (basically rank $d$ measurements analogue to (\ref{eq:inteq2}))
\begin{eqnarray}
\hspace{-.05in}\bl{\textbf{\emph{RdM phase retrieval (RdM PR):}}} \qquad\qquad  \mbox{find} & & \x \nonumber \\
   \mbox{subject to} & & \x^T\cB_{(i)}\x=\bar{\y}_i (= \bar{\x}^T \cB_{(i)}\bar{\x}), \quad i=1,2,\dots,m. \label{eq:ex0inteq4}
\end{eqnarray}

To be able to make a clear parallel with the standard real and complex PR we take $\cB_{(i)}$ as PSD. In general, however, one can also consider the higher rank indefinite $\cB_{(i)}$ as well \cite{Huang19,Huangetal20}. Moreover, (\ref{eq:ex0inteq4}) can also be rewritten as
\begin{eqnarray}
 \mbox{find} & & X \nonumber \\
   \mbox{subject to} & & \tr \lp \cB_{(i)}X  \rp =\bar{\y}_i (= \bar{\x}^T \cB_{(i)}\bar{\x}), \quad i=1,2,\dots,m
   \nonumber \\
    & & X^T=X\geq 0, \text{rank}(X)=1. \label{eq:ex0inteq4a0}
\end{eqnarray}
Problems of this type  have been extensively studied throughout the low rank recovery literature \cite{Recht10,Teetal16,Zheng15,Chi19,Ge17,Park17,Bhoj16,Carlsson20,CapSto24}. Often X's rank is taken to be proportional to $n$ and the so-called  nuclear norm relaxations heuristics have shown a great theoretical  promise. Despite being provably polynomial, their practical running is often viewed as large scale computationally unfriendly. Given that $\text{rank}(X)=1$ in PR and that PR usually deals with problem sizes on the order of a few tens of thousands, we here focus on the above mentioned dPRs as computationally more convenient alternatives.

\subsection{Relevant PR prior work}
\label{sec:relwork}

\subsubsection{Algorithmic methods}
\label{sec:algmethods}

While the rank $d$ structure of PR measurements is likely to affect associated theoretical and algorithmic analytical results, it leaves the choice of the employed algorithmic techniques to a large degree unaffected. In other words, many of the algorithms applicable  in the standard rank 1 ($d=1$) scenario can be reutilized for larger $d$'s without much of a change. Previewing some of the most important algorithmic strategies typically employed in rank 1 scenario is therefore necessary and we do so next (for more detailed expositions, we also refer to \cite{Stojnicphretreal24,Stojnicphretinit24}).

\noindent $\star$ \underline{\emph{Convex approaches}:}  Since formulation in (\ref{eq:ex0inteq4a0}) has  matrix rank-1 constraint as a key non-convexity feature, standard optimization  techniques assume its full rank relaxation (see  \cite{CandesSV13,CandesESV13} for the so-called Phaselift and \cite{WaldspurgerdM15}  for a closely related Phasecut relaxation). As \cite{CandesSV13,CandesESV13} showed, when $m=O(n\log(n))$ Phaselift exactly and stably solves standard PR (see, e.g., \cite{PHand17} for further robustness discussion and \cite{CandesL14} for order optimally improved sample complexity, $m=O(n)$; for utilization of Fourier instead of Gaussian measurements see, e.g., \cite{GFK17,CandesLS15b}). Rank dropping  SDP (semi-definite programming) relaxations perform well but their practical implementation becomes computationally challenging as dimensions grow. The search for numerically more efficient alternatives resulted in appearance of the so-called PhaseMax \cite{GoldsteinS18,BahmaniR16} which was shown to be successful provided that the starting point of the algorithm (initializer) is sufficiently close to the true solution (see \cite{HandV16,ChenCandes17} for further discussion on  initializers and \cite{SalehiAH18} for precise PhaseMax analysis; see, also \cite{DhifallahTL17} where PhaseLamp variant was introduced and shown to have better properties than PhaseMax).

Convex methods allow an easy incorporation of further signal structuring (say sparsity, positivity, low-rankness, and so on). One simply follows usual compressed sensing (CS) adaptations (see, e.g. \cite{Ohlsen12,LiVor13,KeungRTi17}). For example, \cite{LiVor13} shows that $m=O(k^2\log(n))$ suffices for successful PR recovery of a $k$-sparse $\bar{\x}$ which is (even in  nonlinear $m\ll n$ regime) weaker than the corresponding  $m=O(k\log(n))$  CS result. Adaptations of standard convex techniques are possible when additional structuring is present  (for two-stage approaches and randomized Kaczmarz upgrades see \cite{JaganathanOH17,IwenVW17} and \cite{TV19,KWei15}, respectively).

\noindent $\star$ \underline{\emph{Non-convex approaches}:} Early PR algorithmic approaches relied on the alternating minimization \cite{Gerch72,Fienup82,Fienup87}. A big breakthrough, however, arrived when \cite{CandesLS15} uncovered that the so-called Wirtinger flow (a modified gradient descent) often significantly outperforms all other known techniques (a closely related amplitude flow variant was considered in, e.g., \cite{WangGE18}; further modifications, including the thresholded Wirtinger flow  and stochastic gradient can be found in  \cite{CLM16} and  \cite{MUZ21}, respectively). Moreover,   theoretical analysis of  \cite{CandesLS15}  revealed that such a practical superiority even supports favorable scaling  $m=O(n\log(n))$  (the sparse variants of Wirtinger and
amplitude flow \cite{YuanWW19,WangGE18} allow for $m=O(k^2\log(n))$; however, \cite{Soltanolkotabi19} in a way breaks the $k^2$ barrier and shows that a local Wirtinger flow convergence can be achieved with $m=O(k\log(n))$ thereby matching CS predictions and in a way contrasting strong objective deviations from convexity \cite{SunQW18,HandLV18}). ADMM  (alternating direction method of multipliers or Douglas-Rachford within the PR context), was introduced in \cite{FannjiangZ20} and shown to perform fairly well (for further  theoretical characterizations of classical alternating minimization projections, see also \cite{Netrapalli0S15,Waldspurger18}).

After deep learning (DL) had shown a great promise within compressed sensing
\cite{LeiJDD19,JordanD20,BoraJPD17,DaskalakisRZ20a}, \cite{HandLV18} put forth Deep (sparse) phase retrieval (DPR) concept. Empirical results on par with those obtained in \cite{BoraJPD17} indicated a strong DL potential within PR context. In fact, \cite{HandLV18} also proved that constant expansion $d$-layer nets achieve superior performance with $m=O(kd\log(n))$, which for a constant $d$ allows breaking $k^2\log(n)$ barrier and achieving desired  CS prediction, $k\log(n)$. To ensure fairness, certain deep learning limitations need to be mentioned as well. Allowed dimensions are typically restrictive, nonzero errors are often impossible to achieve, frequent (time-consuming) retraining is sometimes needed, and precise theoretical characterizations are much harder to obtain  (for a recent progress in this direction, see, \cite{MBBDN23,Stojnicinjdeeprelu24,Stojnicinjrelu24}). Keeping all of this in mind, utilization of deep nets within the PR context is still a bit distant from expected  generic superiority. While awaiting for further DL progress,
the approximate message passing (AMP) algorithms \cite{DonMalMon09} have shown a great promise. These methods and their associated analyses may often heavily depend on model assumptions, but their excellent empirical performance is presented in \cite{SchniterR15}.

\subsubsection{Initializers of non-convex methods and types of performance analysis}
\label{sec:algmethods1}

\noindent $\star$  \underline{\emph{Initializers:}}  Almost all of the above non-convex techniques require a good starting point -- \emph{initializer}.
Consequently a fairly strong theory related to initializers has been developed in recent years. The so-called spectral initializers appeared within the PR context as integral components of alternating minimization \cite{Netrapalli0S15} (for related and earlier considerations see \cite{Li92,Netrapalli0S13}). \cite{CandesLS15} proposed starting Wirtinger flow with simple diagonal spectral initializers. \cite{LuL17}  considered a more general nonnegative diagonal variant  and \emph{precisely} characterized ``overlap vs oversampling ratio'' (for further extensions, see also \cite{LuLi20}). A sharp transition between zero and nonzero overlap phases was uncovered directly impacting performance of non-convex PR algorithms. Negative diagonals as well as complex domains were incorporated in \cite{MondelliM19} where the concept of ``weak threshold'' (critical $\alpha$ below which no diagonal spectral initializer achieves  nonzero overlap) was introduced and \emph{precisely} determined. The spectral  preprocessing from \cite{MondelliM19} was proven in \cite{LuoAL19} as optimal for any $\alpha$ (i.e., not only for the weak threshold). Gaussian measurements related results from \cite{LuL17,MondelliM19}, were complemented by practically often more attractive orthogonal ones in \cite{MaDXMW21} where an Expectation Propagation paradigm was utilized to precisely characterize associated ``overlap vs oversampling ratio'' (for a study of more general measurement models via statistical physics tools, see also \cite{AubinLBKZ20,MaillardKLZ21}). Results from  \cite{MaDXMW21} were established as mathematically rigorous in \cite{DudejaB0M20}.

\noindent $\star$  \underline{\emph{``Qualitative'' vs ``quantitative'' performance characterizations:}}   Most of the theoretical results associated with the above discussed  non-convex methods relate to \emph{qualitative} performance characterizations. They usually provide a strong intuitive hint as how the algorithms behave, but are not fully precise and do not supply sufficiently convincing arguments to support the non-convex methods practical superiority. On the other hand, precise, \emph{quantitative}, characterizations are usually much harder to obtain and the available results are fairly scarce. The few known examples relate to replica analysis and the so-called Bayesian contexts. For example, the vector AMP (VAMP) algorithms \cite{SchniterRF16,RanganSF17} are analyzed via replica methods  in \cite{TakahashiK22} and a similar Bayesian inference context is considered in \cite{MaillardLKZ20} where a large set of replica predictions is confirmed (for a real version of the complex counterpart from \cite{MaillardLKZ20}, see \cite{BarbierKMMZ18}; for a further work on utilization of  Bayesian context  and replica methods, see also  \cite{StrSag25}, where excellent results are obtained for so-called phase-selection algorithm). Since the Bayesian concept heavily relies on the presumed statistical identicalness of $\bar{\x}_i$'s and a perfect statistical prior knowledge of both $\bar{\x}$ and the so-called channel/posterior, it does not conveniently lend itself for a fair comparison with the above discussed methods that do not require such assumptions. On the other hand, it does allow for precise, phase transition type of results substantially superior to the corresponding ones obtained via the simpler qualitative methods.

\subsection{Our contributions}
\label{sec:contrib}

We consider rank $d$ measurements phase retrieval problem (RdM PR) and theoretically study performance of the associated \emph{descending} phase retrieval algorithms (dPR) (dPRs are a generic class of non-convex PR methods that, among others, encompasses all forms of gradient descents). As stated earlier, our focus is on a statistical \emph{linear/proportional} high-dimensional regime where  sample complexity ratio $\alpha=\lim_{n\rightarrow \infty} \frac{m}{n}$ remains constant as $n$ and $m$ grow.

\begin{itemize}
\item Following into the footsteps of \cite{Stojnicphretreal24}, we design a convenient \emph{fundamental  RdM PR optimization} (fd-pro) formulation and recognize relevance of studying its objective (see beginning of Section \ref{sec:2lay}).

  \item Paralleling the development of the rank 1 \emph{Random duality theory}  (RDT) based generic analytical program in \cite{Stojnicphretreal24}, we here develop its rank $d$ analogue and utilize it for statistical studying of fd-pro (see Section \ref{sec:ubrdt}).

   \item A direct connection between fd-pro's objective and dPR's ability to solve PR is established (see Section \ref{sec:algimp}). The impact of oversampling ratio, $\alpha$, on fd-pro's objective (and ultimately dPR) is numerically evaluated through the developed RDT program (see Figures \ref{fig:fig1} and \ref{fig:fig2} for visual illustrations). Emergence of a phase transition phenomenon is observed. In particular, as $\alpha$ increases, fd-pro's objective transitions from a multi to a single local minimum phase (see Figure \ref{fig:fig1}).

   \item Since the so-called strong deterministic duality is not in place, a lack of strong random duality is anticipated as well. This ultimately implies that the plain RDT results from Section \ref{sec:ubrdt} are of strictly bounding nature. A powerful \emph{Lifted} RDT based program is then developed to further lift plain RDT estimates (for theoretical lifted RDT considerations see Section \ref{sec:liftrdt}; for visualization of the ``\emph{lifting effect}'' see Figures \ref{fig:fig3}--\ref{fig:fig5}).

\item The obtained theoretical predictions assume heavy concentrations which in $n\rightarrow\infty$ regime indeed happen. Since dimensions are finite  in practical applications, we introduce the concept of ``\emph{safer compression}'' phase transition adjustment. It suggests that in practical scenarios, one relies on a sample complexity ratios, $\alpha$, slightly higher (say $10-20\%$) than the minimally needed one determined by the phase transitioning behavior. By doing so, the so-called \emph{local jitteriness} of the fd-pro's objective is mitigated (see Figure \ref{fig:fig5}).

 \item For practically highly relevant rank 2 case (the one which effectively emulates the complex phase retrieval), we conduct numerical experiments as well and obtain an excellent agreement with the theoretical predictions. In particular, a log barrier variant of the gradient descent is implemented together with the optimal diagonal spectral initializers from \cite{MondelliM19,LuoAL19,Stojnicphretinit24}. Even in small dimensional context with $n=100$ (where the presence of strong local jitteriness effects is highly likely), the simulated phase transition is fairly close to the safer compression adjusted theoretical prediction (see Figure \ref{fig:fig6}).

\end{itemize}

\section{Analysis of \emph{descending} phase retrieval algorithms (dPR)}
  \label{sec:2lay}

Assuming statistical $\cB_{(i)}$ in (\ref{eq:ex0inteq4}), one can proceed as in \cite{Stojnicphretinit24,Stojnicphretreal24} and (in a fashion similar to \cite{Stojnicinjrelu24,StojnicGardGen13,StojnicICASSP10var,StojnicCSetam09}) connect studying RdM PR theoretical and algorithmic properties to the analysis of the following random optimization problem (rop)
 \begin{eqnarray}
 {\mathcal R}(\cB_{(i)}): \qquad \qquad    \min_{\x} & & \sum_{i=1}^{m} \lp \sqrt{\bar{\y}_i}-\sqrt{\x^T\cB_{(i)}\x}\rp^2. \label{eq:ex0integ5}
\end{eqnarray}
After readjusting dimensions so that  $A\in\mR^{dm\times dn}$ and $\bar{\x}\in\mR^{dn}$, we set
 \begin{eqnarray}
\cB_{(i)} = \sum_{j=0}^{d-1}   A_{jm+i,:}^TA_{jm+i,:}, \label{eq:ex0integ6}
\end{eqnarray}
and rewrite (\ref{eq:ex0integ5}) as
 \begin{equation}
 {\mathcal R}(A): \hspace{.07in}    \min_{\x}  \sum_{i=1}^{m} \lp \sqrt{\bar{\y}_i}-\sqrt{\x^T \lp \sum_{j=0}^{d-1}  A_{jm+i,:}^TA_{jm+i,:} \rp \x}\rp^2
 =
 \min_{\x}  \sum_{i=1}^{m} \lp \sqrt{\bar{\y}_i}-\sqrt{\sum_{j=0}^{d-1} \x^T A_{jm+i,:}^TA_{jm+i,:} \x}\rp^2. \label{eq:ex0integ7}
\end{equation}
It is also relatively easy to see that (\ref{eq:ex0integ7}) can be further rephrased as
 \begin{eqnarray}
 {\mathcal R}(A): \qquad \qquad    \min_{\x,\z} & &
   \sum_{i=1}^{m} \lp \sqrt{\sum_{j=0}^{d-1} \bar{\z}_{jm+1}^2}
   -
   \sqrt{\sum_{j=0}^{d-1} \z_{jm+i}^2  }
   \rp^2
  \nonumber \\
  \mbox{subject to} & &  A\x=\z. \label{eq:ex1a4}
\end{eqnarray}
It will turn out as  convenient to set
 \begin{eqnarray}
\hspace{-.0in}\bl{\textbf{\emph{Fundamental RdM PR optimization (fd-pro):}}} \hspace{.08in} \xi(c,x) \triangleq \min_{\x,\z} & &
   \sum_{i=1}^{m} \lp \sqrt{\sum_{j=0}^{d-1} \bar{\z}_{jm+1}^2}
   -
   \sqrt{\sum_{j=0}^{d-1} \z_{jm+i}^2  }
   \rp^2
\nonumber \\
  \mbox{subject to} & &  A\x=\z \nonumber \\
  & &  A\bar{\x}=\bar{\z} \nonumber \\
  & & \x^T\bar{\x}=x \nonumber \\
  & & \|\x\|_2^2=c. \label{eq:ex1a4a0}
\end{eqnarray}
One can observe  that PR optimization  formulations usually seen throughout the literature are slightly different from (\ref{eq:ex0integ5}), (\ref{eq:ex1a4}), and (\ref{eq:ex1a4a0}). A full analogy with the \emph{squared magnitudes} common practice would be to use $\sum_{j=0}^{d-1}\bar{\z}_{jm+1}^2$ and $\sum_{j=0}^{d-1}\z_{jm+1}^2$ instead of their roots in the objective. In the standard rank 1 measurements contexts ($d=1$), such practice relates to the so=called \emph{intensity} measurements whereas the non-squared magnitudes relate to the so-called \emph{amplitude} measurements. As \cite{Stojnicphretinit24,Stojnicphretreal24} demonstrated, not  much of a conceptual difference between the two options actually exists. The \emph{non-squared} option allows for more elegant and numerically less cumbersome treatment and as such will be our preferable choice in analytical considerations.

From (\ref{eq:ex1a4a0}) one also easily has $\xi(c,x)=\xi(c,-x)$. To avoid phase ambiguity trivialities, we throughout the presentation consider only $x>0$ (all obtained results will automatically hold for corresponding $x<0$ as well). As \cite{Stojnicphretreal24} observed, fd-pro and in particular $\xi(c,x)$ are important mathematical structures behind the study of two key PR problems. \textbf{\emph{(i)}} The first one is the uniqueness of PR's solution and directly corresponds to ensuring that $\xi(1,x)>0$ for $x\neq 1$.\textbf{\emph{(ii)}}  The second one is the PR's practical (algorithmic) solvability and directly relates to increasing/decreasing behavior of function $\xi(c,x)$. To make the presentation neater, we consider a statistical scenario with $A$ comprised of iid standard normals. Moreover, following into the footsteps of \cite{Stojnicphretreal24}, we consider \emph{Random duality theory} (RDT) based statistical analytical program developed therein and discuss how it can be adapted to handle  fd-pro and ultimately RdM PR.

\subsection{Analysis of $\xi(c,x)$ via Random Duality Theory (RDT)}
\label{sec:ubrdt}

We first provide a summary of the main RDT principles developed in a long line of work \cite{StojnicCSetam09,StojnicICASSP10var,StojnicISIT2010binary,StojnicICASSP10block,StojnicRegRndDlt10,StojnicGenLasso10} and then proceed with the discussion related to the implementation of each of these principles within the RdM PR context of interest here.

\vspace{-.0in}\begin{center}
 \begin{tcolorbox}[title={\small Summary of the RDT's main principles} \cite{StojnicCSetam09,StojnicRegRndDlt10}]
\vspace{-.15in}
{\small \begin{eqnarray*}
 \begin{array}{ll}
\hspace{-.19in} \mbox{1) \emph{Finding underlying optimization algebraic representation}}
 & \hspace{-.0in} \mbox{2) \emph{Determining the random dual}} \\
\hspace{-.19in} \mbox{3) \emph{Handling the random dual}} &
 \hspace{-.0in} \mbox{4) \emph{Double-checking strong random duality.}}
 \end{array}
  \end{eqnarray*}}
\vspace{-.2in}
 \end{tcolorbox}
\end{center}\vspace{-.0in}

\noindent To ensure that key results (including both simple and more complicated ones) are clearly visible and easily accessible, we formulate all of them as lemmas or theorems. Also, since we will be adapting results of \cite{Stojnicphretreal24}, we to a large degree parallel the methodology presented therein. However, to ensure that unnecessary repetitions are avoided whenever possible, we skip rehashing already introduced concepts and instead focus on emphasizing those that are different.

\vspace{.1in}

\noindent \underline{1) \textbf{\emph{Algebraic fd-pro characterization:}}}  Due to  rotational invariance of $A$ one can
 without loss of generality rotate $\bar{\x}$ so that it becomes $\bar{\x}=[\|\bar{\x}\|_2,0,\dots,0 ]^T$. That allows to
first rewrite (\ref{eq:ex1a4a0}) as
 \begin{eqnarray}
  \xi(c,x) = \min_{\x,\z} & &
   \sum_{i=1}^{m} \lp \sqrt{\sum_{j=0}^{d-1} \bar{\z}_{jm+1}^2}
   -
   \sqrt{\sum_{j=0}^{d-1} \z_{jm+i}^2  }
   \rp^2
\nonumber \\
  \mbox{subject to} & &  A\x=\z \nonumber \\
  & &  A\bar{\x}=A_{:,1}\|\bar{\x}\|_2=\bar{\z} \nonumber \\
  & & \x^T\bar{\x}=x \nonumber \\
  & & \|\x\|_2^2=c, \label{eq:rdteq0a0a0}
\end{eqnarray}
with $A_{:,i}$ being the $i$-th column of $A$. After noting that restriction $\|\bar{\x}\|_2=1$ brings no loss of generality, one then also has
 \begin{eqnarray}
  \xi(c,x) = \min_{\x,\z} & &
   \sum_{i=1}^{m} \lp \sqrt{\sum_{j=0}^{d-1} A_{jm+1,1}^2}
   -
   \sqrt{\sum_{j=0}^{d-1} \z_{jm+i}^2  }
   \rp^2
\nonumber \\
  \mbox{subject to} & &  A\x=\z \nonumber \\
   & & \x^T\bar{\x}=\x_1\|\bar{\x}\|_2=\x_1= x \nonumber \\
  & & \|\x\|_2^2=c, \label{eq:rdteq0a0}
\end{eqnarray}
After setting $r\triangleq \sqrt{c-x^2}$, (\ref{eq:rdteq0a0}) is further rewritten as
 \begin{eqnarray}
  \xi(c,x) = \min_{\x,\z} & &
     \sum_{i=1}^{m} \lp \sqrt{\sum_{j=0}^{d-1} A_{jm+1,1}^2}
   -
   \sqrt{\sum_{j=0}^{d-1} \z_{jm+i}^2  }
   \rp^2
  \nonumber \\
  \mbox{subject to} & &  A\x=A_{:,1}x + A_{:,2:n}\x_{2:n} = \z \nonumber \\
   & & \sum_{i=2}^{n}\x_i=c-x^2=r^2, \label{eq:rdteq0a1}
\end{eqnarray}
Utilizing the Lagrangian we have
 \begin{eqnarray}
  \xi(c,x) = \min_{\x,\z} \max_{\y}   & &
     \sum_{i=1}^{m} \lp \sqrt{\sum_{j=0}^{d-1} A_{jm+1,1}^2}
   -
   \sqrt{\sum_{j=0}^{d-1} \z_{jm+i}^2  }
   \rp^2
  +\y^TA_{:,1}x + \y^TA_{:,2:n}\x_{2:n} -\y^T \z \nonumber \\
  \mbox{subject to}
     & & \sum_{i=2}^{n}\x_i=c-x^2=r^2. \label{eq:rdteq0a2}
\end{eqnarray}
We then also set $\g^{(0)}=A_{:,1}$ and write  (\ref{eq:rdteq0a2}) in the following more compact form
 \begin{eqnarray}
  \xi(c,x) = \min_{\|\x_{2:n}\|_2=r,\z} \max_{\y}  \lp
       \sum_{i=1}^{m} \lp \sqrt{\sum_{j=0}^{d-1} \lp \g_{jm+i}^{(0)}\rp^2}
   -
   \sqrt{\sum_{j=0}^{d-1} \z_{jm+i}^2  }
   \rp^2
  +\y^T \g^{(0)}x + \y^TA_{:,2:n}\x_{2:n} -\y^T \z \rp. \label{eq:rdteq0a3}
\end{eqnarray}
The above  is a key algebraic characterization of fd-pro. It is summarized in the following lemma together with its implications regarding theoretical solvability of RdM PR.

\begin{lemma} Consider rank $d$ measurements phase retrieval (RdM PR) problem with $dn$ unknowns and sample complexity $m$. Let $A\in\mR^{dm\times dn}$, $\g^{(0)}\triangleq A_{:,1}$, and assume a high-dimensional linear (proportional) regime,
with sample complexity ratio $\alpha=\lim_{n\rightarrow\infty}\frac{m}{n}$. Then, RdM PR is theoretically solvable (i.e., it has a unique (up to a global phase) solution) provided that
 \begin{equation}\label{eq:ta10}
 \forall x\neq 1 \qquad f_{rp}(1,x;A)>0,
\end{equation}
where
\begin{equation}\label{eq:ta11}
f_{rp}(c,x;A) \hspace{-.03in}\triangleq   \hspace{-.05in}\frac{1}{dn} \hspace{-.04in} \min_{\|\x_{2:n}\|_2=r,\z} \max_{\y}  \lp
       \sum_{i=1}^{m} \lp \sqrt{\sum_{j=0}^{d-1} \lp \g_{jm+i}^{(0)}\rp^2}
   -
   \sqrt{\sum_{j=0}^{d-1} \z_{jm+i}^2  }
   \rp^2
  +\y^T \g^{(0)}x + \y^TA_{:,2:n}\x_{2:n} -\y^T \z \rp.
\end{equation}
    \label{lemma:lemma1}
\end{lemma}
\begin{proof}
Follows as an automatic consequence of the preceding discussion and the recognition that (\ref{eq:ta11}) is (\ref{eq:rdteq0a3})  cosmetically  scaled by $\frac{1}{dn}$.
\end{proof}

The optimization on the right hand side of (\ref{eq:ta11}) is the so-called \emph{random primal}. We determine the corresponding \emph{random dual} next.

\vspace{.1in}
\noindent \underline{2) \textbf{\emph{Determining the random dual:}}} As usual within the RDT, the concentration of measure is utilized as well. This basically means that for any fixed $\epsilon >0$,  one has (see, e.g. \cite{StojnicCSetam09,StojnicRegRndDlt10,StojnicICASSP10var})
\begin{equation*}
\lim_{n\rightarrow\infty}\mP_{A}\left (\frac{|f_{rp}(c,x;A)-\mE_{A}(f_{rp}(c,x;A)|}{\mE_{A}(f_{rp}(c,x;A)}>\epsilon\right )\longrightarrow 0.\label{eq:ta15}
\end{equation*}
 The following, so-called random dual, theorem is another critically important RDT ingredient.
\begin{theorem} Assume the setup of Lemma \ref{lemma:lemma1} with the elements of $A\in\mR^{dm\times dn}$ ($\g^{(0)}\in\mR^{dm\times 1}$ and $A_{:,2:dn}\in\mR^{dm\times dn-1 }$), $\g^{(1)}\in\mR^{dm\times 1}$, and  $\h^{(1)}\in\mR^{(dn-1)\times 1}$  being iid standard normals. For two positive scalars $c$ and $x$  ($0\leq x \leq c$) set $r\triangleq \sqrt{c-x^2}$. Let
\vspace{-.0in}
\begin{eqnarray}
\cG  \hspace{-.12in} & \triangleq &  \hspace{-.12in} \lp A,\g^{(1)},\h^{(1)}\rp = \lp\g^{(0)},A_{:,2:dn},\g^{(1)},\h^{(1)}\rp  \nonumber \\
\phi(\x,\z,\y) \hspace{-.12in}& \triangleq &\hspace{-.12in}
        \sum_{i=1}^{m} \lp \sqrt{\sum_{j=0}^{d-1} \lp \g_{jm+i}^{(0)}\rp^2}
   -
   \sqrt{\sum_{j=0}^{d-1} \z_{jm+i}^2  }
   \rp^2
\hspace{-.05in}  +\y^T \g^{(0)}x   +  \y^T \g^{(1)}\|\x_{2:dn}\|_2 + \lp \x_{2:dn}  \rp^T\h^{(1)}\|\y\|_2  -\y^T  \z
\nonumber \\
 f_{rd}(c,x;\cG) \hspace{-.12in} & \triangleq &  \hspace{-.12in}
\frac{1}{dn}  \min_{\|\x_{2:dn}\|_2=r,\z} \max_{\|\y\|_2=r_y,r_y>0}  \phi(\x,\z,\y)
  \nonumber \\
 \phi_0 \hspace{-.12in} & \triangleq &  \hspace{-.12in} \lim_{n\rightarrow\infty} \mE_{\cG} f_{rd}(c,x;\cG).\label{eq:ta16}
\vspace{-.0in}\end{eqnarray}
One then has \vspace{-.02in}
\begin{eqnarray}
  \lim_{n\rightarrow\infty}\mP_{ A } (f_{rp} (c,x; A )   >  \phi_0)\longrightarrow 1,\label{eq:ta17a0}
\end{eqnarray}
and
\begin{eqnarray}
\hspace{-.07in}(\phi_0  > 0)   &  \Longrightarrow  & \lp \lim_{n\rightarrow\infty}\mP_{\cG}\lp \frac{\xi(c,x)}{dn} = f_{rd}(c,x;\cG) >0 \rp \longrightarrow 1\rp
\quad  \Longrightarrow \quad \lp \lim_{n\rightarrow\infty}\mP_{ A } (f_{rp} (c,x; A )   >0)\longrightarrow 1 \rp  \nonumber \\
& \Longrightarrow & \lp \lim_{n\rightarrow\infty}\mP_{A} \lp \mbox{PR is (uniquely) solvable} \rp \longrightarrow 1\rp.\label{eq:ta17}
\end{eqnarray}
 \label{thm:thm1}
\end{theorem}\vspace{-.17in}
\begin{proof}
It is an immediate consequence of Theorem 1 in \cite{Stojnicphretreal24} which follows through conditioning on $\g^{(0)}$ and application of the Gordon's comparison theorem (see, e.g., Theorem B in \cite{Gordon88}). Gordon's theorem can be obtained as a special case of a series of Stojnic's results from \cite{Stojnicgscomp16,Stojnicgscompyx16} (see Theorem 1, Corollary 1, and Section 2.7.2 in \cite{Stojnicgscomp16} as well as Theorem 1, Corollary 1, and Section 2.3.2 in \cite{Stojnicgscompyx16}).
\end{proof}

\vspace{.1in}
\noindent \underline{3) \textbf{\emph{Handling the random dual:}}} To handle the above random dual we rely on \cite{Stojnicphretreal24} (the results of \cite{Stojnicphretreal24} were obtained following the methodologies invented in \cite{StojnicCSetam09,StojnicICASSP10var,StojnicISIT2010binary,StojnicICASSP10block,StojnicRegRndDlt10}). After solving the optimizations over $\x$ and $\y$ we first obtain from (\ref{eq:ta16})
\begin{equation}
 f_{rd}(c,x;\cG) \hspace{-.03in}= \hspace{-.05in}
\frac{1}{dn}  \min_{\z} \max_{r_y>0}
\lp
        \sum_{i=1}^{m} \lp \sqrt{\sum_{j=0}^{d-1} \lp \g_{jm+i}^{(0)}\rp^2}
   -
   \sqrt{\sum_{j=0}^{d-1} \z_{jm+i}^2  }
   \rp^2
  +\|\g^{(0)}x   +  \g^{(1)}r -\z\|_2 r_y - \|\h^{(1)} \|_2 r r_y \rp.
\label{eq:hrd1}
 \end{equation}
We find it useful to set
\begin{eqnarray}
\g_{(i)}^{(0)}  & \triangleq  &
\begin{bmatrix}
  \g_{i}^{(0)}
 & \g_{m+i}^{(0)}
 & \g_{2m+i}^{(0)}
 & \dots & \g_{(d-1)m+i}^{(0)}
\end{bmatrix}^T \nonumber \\
\g_{(i)}^{(1)}  & \triangleq  &
\begin{bmatrix}
  \g_{i}^{(1)}
 & \g_{m+i}^{(1)}
 & \g_{2m+i}^{(1)}
 & \dots & \g_{(d-1)m+i}^{(1)}
\end{bmatrix}^T \nonumber \\
\z_{(i)} & \triangleq  &
\begin{bmatrix}
  \z_{i}
 & \z_{m+i}
 & \z_{2m+i}
 & \dots & \z_{(d-1)m+i}
\end{bmatrix}^T.
 \label{eq:hrd1a0}
 \end{eqnarray}
A combination of (\ref{eq:hrd1}) and (\ref{eq:hrd1a0}) then gives
\begin{eqnarray}
 f_{rd}(c,x;\cG)
\hspace{-.09in}  &  =  & \hspace{-.09in}
\frac{1}{dn}  \min_{\z} \max_{r_y>0}
\lp
        \sum_{i=1}^{m} \lp \|  \g_{(i)}^{(0)}   \|_2
   -
\|  \z_{(i)}   \|_2
   \rp^2
  +\|\g^{(0)}x   +  \g^{(1)}r -\z\|_2 r_y - \|\h^{(1)} \|_2 r r_y \rp
  \nonumber \\
\hspace{-.05in}  &  =  &\hspace{-.09in}
\frac{1}{dn}  \min_{\z} \max_{r_y>0}
\lp
        \sum_{i=1}^{m} \lp \|  \g_{(i)}^{(0)}   \|_2
   -
\|  \z_{(i)}   \|_2
   \rp^2
  +\|\g^{(0)}x   +  \g^{(1)}r -\z\|_2^2 r_y - \|\h^{(1)} \|_2^2 r^2 r_y \rp
  \nonumber \\
\hspace{-.05in}  &  \geq  &\hspace{-.09in}
\frac{1}{dn}  \max_{r_y>0} \min_{\z}
\lp
        \sum_{i=1}^{m} \lp \|  \g_{(i)}^{(0)}   \|_2
   -
\|  \z_{(i)}   \|_2
   \rp^2
  +\|\g^{(0)}x   +  \g^{(1)}r -\z\|_2^2 r_y - \|\h^{(1)} \|_2^2 r^2 r_y \rp
  \nonumber \\
\hspace{-.05in}   &  =   &\hspace{-.09in}
\frac{1}{dn}  \max_{r_y>0} \min_{\z_i}
 \lp
\sum_{i=1}^{m}
\lp \lp \|\g_{(i)}^{(0)}\|_2 - \|\z_{(i)}\|_2  \rp^2
  +\|\g_{(i)}^{(0)}x   +  \g_{(i)}^{(1)}r -\z_{(i)}\|_2^2 r_y \rp
  - \|\h^{(1)} \|_2^2 r^2 r_y \rp
    \nonumber \\
 \hspace{-.05in}  &  =   & \hspace{-.09in}
\frac{1}{dn}  \max_{r_y>0} \min_{\z_i}
 \lp
\sum_{i=1}^{m}
\lp \lp \|\g_{(i)}^{(0)}\|_2 - \|\z_{(i)}\|_2  \rp^2
  +\lp \| \g_{(i)}^{(0)}x   +  \g_{(i)}^{(1)}r \|_2  - \|\z_{(i)}\|_2   \rp^2 r_y \rp
  - \|\h^{(1)} \|_2^2 r^2 r_y \rp. \nonumber \\
\label{eq:hrd2}
 \end{eqnarray}
 The above optimization is structurally identical to the one given in (16) in \cite{Stojnicphretreal24}. One can then  repeat all the steps between (17) and (27) in \cite{Stojnicphretreal24}. Analogously to (20) in \cite{Stojnicphretreal24}, we first finds for the optimal $\|z_{(i)}\|_2$
\begin{eqnarray}
\|\hat{\z}_{(i)}\|_2= \frac{1}{1+r_y} \lp   \|\g_{(i)}^{(0)}\|_2 + \|\g_{(i)}^{(0)}x   +  \g_{(i)}^{(1)}r \|_2 r_y \rp.
\label{eq:hrd7a1}
 \end{eqnarray}
After setting
\begin{eqnarray}
f_q & \triangleq & \mE_{\cG}\lp \|\g_{(i)}^{(0)}\|_2 - \|\g_{(i)}^{(0)}x   +  \g_{(i)}^{(1)}r\|_2  \rp^2,
\label{eq:hrd7a2a0}
 \end{eqnarray}
one then analogously to (27) in \cite{Stojnicphretreal24} obtains
\begin{eqnarray}
 \phi_0 & \geq &
    \max_{r_y>0}
  \lp
  \frac{\alpha}{d}
\frac{r_y}{1+r_y} f_q
  -  r^2 r_y \rp.
\label{eq:hrd7a4}
 \end{eqnarray}
It is not necessarily relevant for what we consider here, but we do mention that (as in \cite{Stojnicphretreal24}) the inequality signs in (\ref{eq:hrd2}) and (\ref{eq:hrd7a4}) can trivially be replaced with equalities. After solving the residual optimization over $r_y$  one finally finds
\begin{eqnarray}
 \phi_0 & \triangleq & \lim_{n\rightarrow\infty} \mE_{\cG} f_{rd}(\cG)
  \geq
\max\lp \sqrt{\frac{\alpha}{d} f_q}
  -  r,0\rp^2.
\label{eq:hrd7a7}
 \end{eqnarray}
(\ref{eq:hrd7a2a0})  and (\ref{eq:hrd7a7}) are basically sufficient to ultimately handle the random dual. From (\ref{eq:hrd7a2a0})we then have
\begin{eqnarray}
f_q & \triangleq & \mE_{\cG}\lp \|\g_{(i)}^{(0)}\|_2 - \|\g_{(i)}^{(0)}x   +  \g_{(i)}^{(1)}r\|_2  \rp^2
\nonumber \\
& = &
d(1+c) -2 \mE_{\cG}\lp \|\g_{(i)}^{(0)}\|_2 \|\g_{(i)}^{(0)}x   +  \g_{(i)}^{(1)}r\|_2  \rp
\nonumber \\
& = &
d(1+c) -2 \mE_{\cG}\lp \|\g_{(i)}^{(0)}\|_2 \| U\g_{(i)}^{(0)}x   +  U\g_{(i)}^{(1)}r\|_2  \rp
\nonumber \\
& = &
d(1+c) -2r \mE_{\cG}\lp \|\g_{(i)}^{(0)}\|_2 \left \| \frac{1}{r}\g_{(i)}^{(0,u)}x   +  \g_{(i)}^{(1)}\right \|_2  \rp,
\label{eq:hrd7a2a0a0}
 \end{eqnarray}
 where $U^TU=I$ and $U$ is such that
 \begin{eqnarray}
\g_{(i)}^{(0,u)} = U \g_{(i)}^{(0)}  & = &
\begin{bmatrix}
  \frac{\| \g_{(i)}^{(0)}  \|_2}{\sqrt{d}}
 &   \frac{\| \g_{(i)}^{(0)}  \|_2}{\sqrt{d}}
 &   \frac{\| \g_{(i)}^{(0)}  \|_2}{\sqrt{d}}
 & \dots
  &   \frac{\| \g_{(i)}^{(0)}  \|_2}{\sqrt{d}}
\end{bmatrix}^T.
  \label{eq:hrd7a2a0a1}
 \end{eqnarray}
Conditioned on $\| \g_{(i)}^{(0)}  \|_2$, $\| \frac{1}{r}\g_{(i)}^{(0,u)}x   +  \g_{(i)}^{(1)}\|_2$ is a non-central chi distributed random variable with $d$ degrees of freedom and the mean of each of the $d$ constituting iid Gaussians $\mu=\frac{\| \g_{(i)}^{(0)}  \|_2 x }{r\sqrt{d}}$. One then sets
\begin{eqnarray}
\lambda_d = \mu\sqrt{d},   \label{eq:hrd7a2a0a32}
 \end{eqnarray}
 and finds
\begin{eqnarray}
  \mE_{\g_{(i)}^{(1)}}  \left \| \mu   +  \g_{(i)}^{(1)}\right \|_2
  = L_{\frac{1}{2}}^{\frac{d}{2}-1} \lp -\frac{\lambda_d^2}{2}\rp,
   \label{eq:hrd7a2a0a3}
 \end{eqnarray}
where $L_{\frac{1}{2}}^{\frac{d}{2}-1} (\cdot)$ is a Laguerre function. As $u=\| \g_{(i)}^{(0)}  \|_2$ itself is a central chi square distributed with $d$ degrees of freedom, one has for its pdf
\begin{eqnarray}
 p(u) = \frac{u^{d-1} e^{-\frac{x^2}{2}}}{2^{\frac{d}{2}-1 }  \Gamma\lp\frac{d}{2}\rp }, u\geq 0.
   \label{eq:hrd7a2a0a4}
 \end{eqnarray}
A combination of (\ref{eq:hrd7a2a0a0}), (\ref{eq:hrd7a2a0a3}), and (\ref{eq:hrd7a2a0a4}) then gives
\begin{eqnarray}
f_q
& = &
d(1+c) -2r \mE_{\cG}\lp \|\g_{(i)}^{(0)}\|_2 \left \| \frac{1}{r}\g_{(i)}^{(0,u)}x   +  \g_{(i)}^{(1)}\right \|_2  \rp
\nonumber \\
& = &
d(1+c) -2r \mE_{\|\g_{(i)}^{(0)}\|_2 }  \lp \|\g_{(i)}^{(0)}\|_2
\mE_{\|\g_{(i)}^{(1)}\|_2 }
 \left \| \frac{1}{r}\g_{(i)}^{(0,u)}x   +  \g_{(i)}^{(1)}\right \|_2  \rp\nonumber \\
& = &
d(1+c) -2r \mE_{\|\g_{(i)}^{(0)}\|_2 }  \lp \|\g_{(i)}^{(0)}\|_2
\mE_{\|\g_{(i)}^{(1)}\|_2 }
 \left \| \mu   +  \g_{(i)}^{(1)}\right \|_2  \rp
 \nonumber\\
 & = &
d(1+c) -2r \mE_{\|\g_{(i)}^{(0)}\|_2 }  \lp \|\g_{(i)}^{(0)}\|_2
 L_{\frac{1}{2}}^{\frac{d}{2}-1} \lp -\frac{\lambda_d^2}{2}\rp \rp
  \nonumber\\
 & = &
d(1+c) -2r \mE_{\|\g_{(i)}^{(0)}\|_2 }  \lp \|\g_{(i)}^{(0)}\|_2
 L_{\frac{1}{2}}^{\frac{d}{2}-1} \lp -\frac{\| \g_{(i)}^{(0)}  \|_2^2 x^2 }{2r^2} \rp \rp
   \nonumber\\
 & = &
d(1+c) -2r \int_{0}^{\infty}
u L_{\frac{1}{2}}^{\frac{d}{2}-1} \lp -\frac{u^2 x^2 }{2r^2} \rp
 \frac{u^{d-1} e^{-\frac{x^2}{2}}}{2^{\frac{d}{2}-1 }  \Gamma\lp\frac{d}{2}\rp }du.
\label{eq:hrd7a2a0a5}
 \end{eqnarray}
The elegance and simplicity of the numerical evaluations depends on $d$. For the concreteness we take $d=2$ as it corresponds to a direct emulation of the complex phase retrieval. Then for $x_{\lambda} =\frac{\lambda_d^2}{2}$ the above has a bit more convenient form
\begin{eqnarray}
f_q
& = &
 2(1+c) -2r \int_{0}^{\infty}
u L_{\frac{1}{2}}^{0} \lp -\frac{u^2 x^2 }{2r^2} \rp
 u e^{-\frac{u^2}{2}}du
\nonumber \\
 & = &
 2(1+c) -2r \int_{0}^{\infty}
u
\sqrt{\frac{\pi}{2}} \lp (x_{\lambda}+1) \cI_0 \lp \frac{x_{\lambda}}{2} \rp  +x_{\lambda} \cI_1 \lp \frac{x_{\lambda}}{2}\rp  \rp  e^{-\frac{x_{\lambda}}{2}}
 u e^{-\frac{u^2}{2}}du
 \nonumber \\
 & = &
 2(1+c) -2r \int_{0}^{\infty}
u
\sqrt{\frac{\pi}{2}} \lp \lp \frac{u^2 x^2 }{2r^2} +1 \rp \cI_0 \lp \frac{u^2 x^2 }{4r^2}  \rp  + \frac{u^2 x^2 }{2r^2}  \cI_1 \lp \frac{u^2 x^2 }{4r^2}  \rp  \rp  e^{-\frac{u^2 x^2 }{4r^2}   }
 u e^{-\frac{u^2}{2}}du,
\label{eq:hrd7a2a0a6}
 \end{eqnarray}
where $\cI_0(\cdot)$ and $\cI_1(\cdot)$ are  modified Bessel functions of the first kind. Plugging the value from (\ref{eq:hrd7a2a0a6}) in (\ref{eq:hrd7a7}) completes handling of the random dual for $d=2$. For different $d$ one can proceed with similar numerical evaluations. Since the analytical expression  similar to the ones given in (\ref{eq:hrd7a2a0a6}) get a bit more cumbersome as $d$ increases we skip stating them explicitly but do mention that the underlying numerical evaluations can be done without much of a problem.


  \vspace{.1in}
\noindent \underline{4) \textbf{\emph{Double checking the strong random duality:}}}   The last step of the RDT machinery assumes double checking whether the strong random duality holds. As in \cite{Stojnicphretreal24},  a lack of deterministic strong duality does not allow for a utilization of the reversal considerations from \cite{StojnicRegRndDlt10}  and the strong random duality is not in place. This practically  implies that the above results are strict lower bounds on the fd-pro's scaled objective, $\xi(c,x)$ .

\subsection{Numerical evaluations and algorithmic implications}
\label{sec:algimp}

As discussed in great detail in \cite{Stojnicphretreal24}, behavior of $\xi(c,x)$ is directly related to the performance of the descending phase retrieval algorithms (dPR). The above results allow to evaluate $\phi_0$ and by doing so to estimate $\xi(c,x)$. In Figure \ref{fig:fig1} we show $\sqrt{\phi_0}$ as a function of overlap $x$ for several different values of sample complexity ratio $\alpha$. We keep $c=1$ fixed as the optimal solution of RdM PR has norm 1. Clearly, recovery in phase retrieval is successful when $c=1$, $x=1$, and consequently  fd-pro objective $\xi(1,1)=0$. For any descending algorithm to be successful one must have that $\sqrt{\phi_0}$ as function of $x$ (for $c=1$) has a single minimum achieved for $c=1$. As can be seen from Figure \ref{fig:fig1}, this does not happen for $\alpha=2.4$ or $\alpha=2.6$ since one has another minimum ar $x=0$. This effectively indicates that the dPR might get trapped if initialized unfavorably. On the other hand, starting from $\alpha\approx 2.79$, the curve sufficiently flattens out and the $x=0$ minimum disappears. This effectively implies that RDT predicts the dPR phase transitioning sample complexity ratio to be $\alpha\approx 2.79$. Figure  \ref{fig:fig2} shows that this indeed is the case as long as the dPRs are run in such a fashion that $c\leq 1$ (if $c\geq 1$ is allowed then dPR can still be successful but RDT predicts that generically one can not guarantee that). In particular, Figure  \ref{fig:fig2} shows $\sqrt{\phi_0}$ as a function of overlap $x$ for several different values of $c$ for critical  sample complexity ratio $\alpha\approx 2.79$. As can be seen, as $c$ decreases curves are ``less flat'' (i.e., decrease more rapidly). This effectively indicates that $c=1$ is the critical value of the norm of the optimizing variable in RdM PR and from this point on we typically focus on such scenarios (for an extended discussion regarding the shapes and relevance of the associated so-called \emph{parametric manifolds}, ${\mathcal {PM}}$, which simultaneously account for all possible changes in $c$ and $x$, see \cite{Stojnicphretreal24}).

\begin{figure}[h]
\centering
\centerline{\includegraphics[width=1\linewidth]{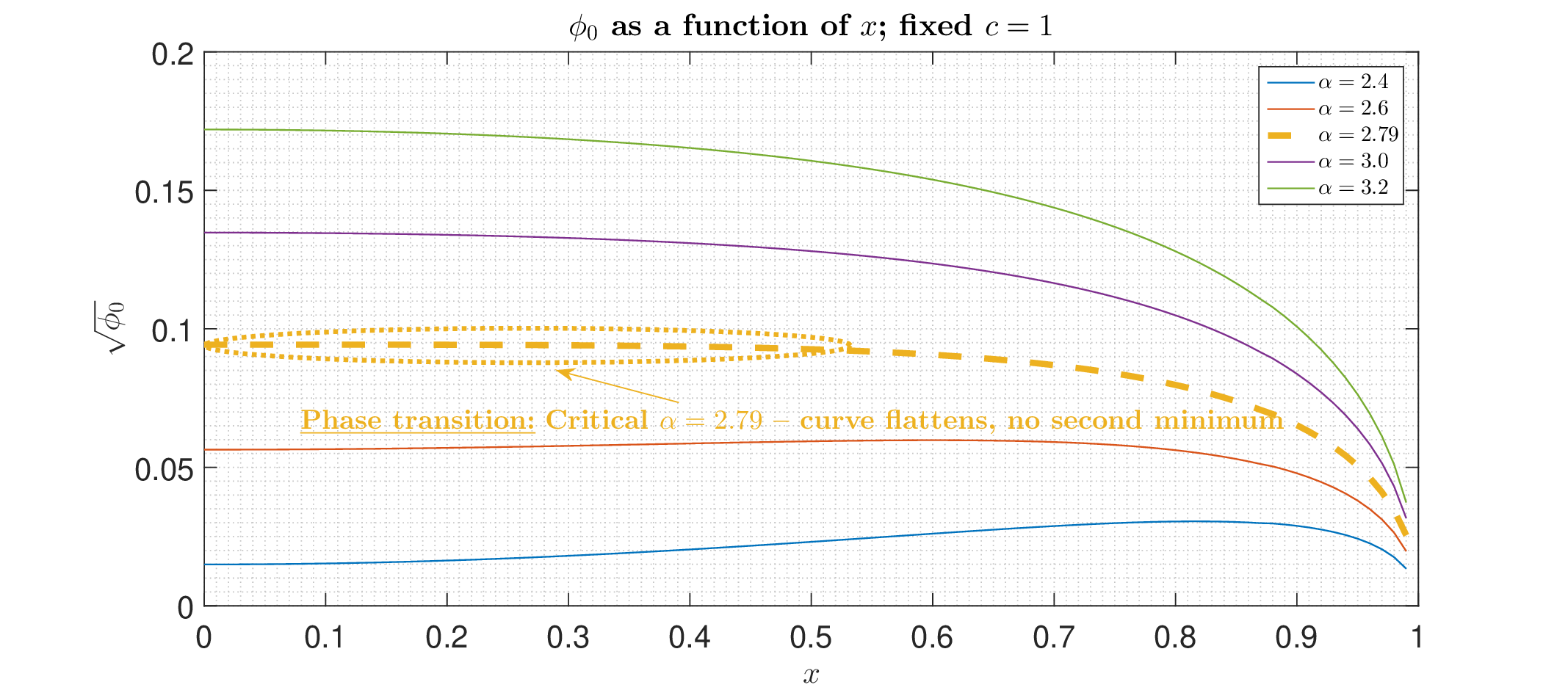}}
\caption{$\phi_0$ as a function of $x$ for different values of $\alpha$; fixed $c=1$; plain RDT}
\label{fig:fig1}
\end{figure}

In addition to the above mentioned strong preference for utilization of practical algorithms that can ensure  $c\leq 1$, two other points need to be kept in mind. \textbf{\emph{(i)}} First, the above theoretical results assume strong concentrations which indeed happen when  $n\rightarrow\infty$. As in practical running  $n$ is limited, the shapes of the curves shown in Figures \ref{fig:fig1} and \ref{fig:fig2} might in  reality be slightly different. While they are likely to maintain the shown \emph{global} increasing/decreasing tendencies, they are also likely to exhibit \emph{local jitteriness} effects. Those on the other hand might sometimes be sufficiently large to create local wells that can ultimately trap descending  algorithms. To combat such eventualities it is often practically safer to follow a simple ``\emph{safer compression}'' rule of thumb which suggests to operate in sample complexity ratio regimes that are slightly (say $10-15\%$) above the phase transitioning prediction.
\begin{figure}[h]
\centering
\centerline{\includegraphics[width=1\linewidth]{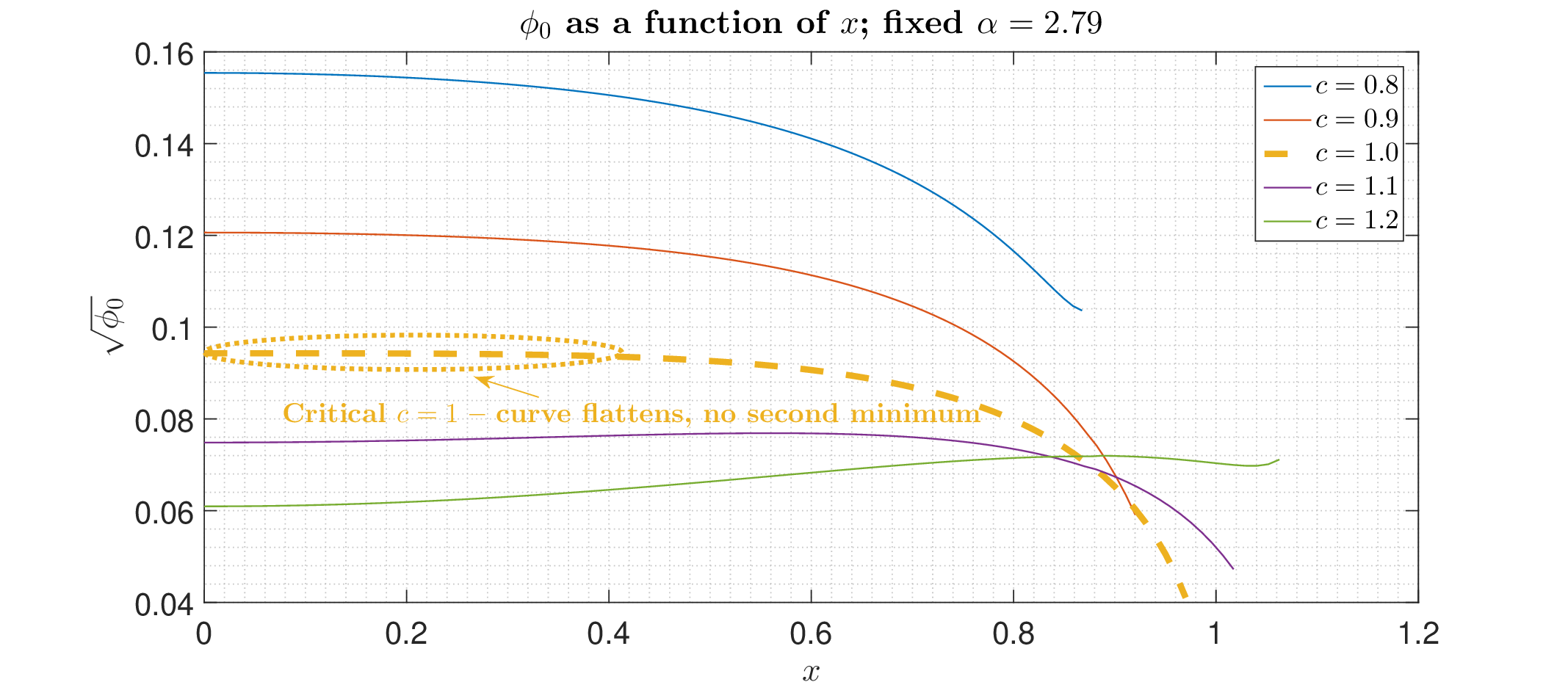}}
\caption{$\phi_0$ as a function of $x$ for different values of $c$; fixed $\alpha=2.79$; plain RDT}
\label{fig:fig2}
\end{figure}
\textbf{\emph{(ii)}} The second point relates to the fact that he strong RDT is not in place which implies a need to keep in mind that the given estimates might have to be adjusted. In particular, the sample complexity ratio estimate  is likely to go down if one implements \emph{lifted} RDT (this, however, in no way affects the above conceptual conclusions; it may only have a favorable  effect on the concrete value of the minimal needed $\alpha$). We discuss the effect of the lifted RDT next and ultimately show that the above (plain RDT)  $\alpha$  estimates can indeed be lowered.

\section{Lifted RDT}
 \label{sec:liftrdt}

A lack of strong random duality implies that the above bounds are \emph{strict} which then allows for their further lifting. Recent development of \emph{fully lifted} (fl) RDT \cite{Stojnicflrdt23} allows the ultimate lifting and a precise characterization of the optimal $\alpha$ values. However, implementation of the fl RDT heavily relies on conducting a sizeable set of numerical evaluations. When it comes to problems of our interest here, these quickly get fairly time consuming and simpler alternatives might be a better choice. Drawing the analogy with \cite{Stojnicphretreal24}, we here consider a bit more convenient,  \emph{partially lifted} (pl) RDT variant \cite{StojnicLiftStrSec13,StojnicMoreSophHopBnds10,Stojnicinjdeeprelu24}. As we will soon see, despite being simpler and numerically less intensive, pl RDT turns out to be able to provide a rather strong improvement over the plain RDT.

As discussed on many occasions in  \cite{Stojnicinjdeeprelu24,StojnicLiftStrSec13,StojnicMoreSophHopBnds10}, the pl RDT relies on similar principles as the plain RDT. However, several subtle technical novelties allow that it makes a substantial improvement over the plain RDT. The key among them is consideration of the so-called \emph{partially lifted random dual}. It is introduced through the following theorem  (basically a partially lifted analogue to Theorem \ref{thm:thm1}).

\begin{theorem} Assume the setup of Theorem \ref{thm:thm1} with the elements of $A\in\mR^{m\times n}$ ($\g^{(0)}\in\mR^{m\times 1}$ and $A_{:,2:n}\in\mR^{m\times (n-1)}$), $\g^{(1)}\in\mR^{m\times 1}$, and  $\h^{(1)}\in\mR^{(n-1)\times 1}$  being iid standard normals. Consider two positive scalars $c$ and $x$  ($0\leq x \leq c$) and set $r\triangleq \sqrt{c-x^2}$. Let $c_3>0$ and
\vspace{-.0in}
\begin{eqnarray}
\cG  \hspace{-.12in} & \triangleq &  \hspace{-.12in} \lp A,\g^{(1)},\h^{(1)}\rp = \lp\g^{(0)},A_{:,2:dn},\g^{(1)},\h^{(1)}\rp  \nonumber \\
\phi(\x,\z,\y) \hspace{-.12in}& \triangleq &\hspace{-.12in}
        \sum_{i=1}^{m} \lp \sqrt{\sum_{j=0}^{d-1} \lp \g_{jm+i}^{(0)}\rp^2}
   -
   \sqrt{\sum_{j=0}^{d-1} \z_{jm+i}^2  }
   \rp^2
\hspace{-.05in}  +\y^T \g^{(0)}x   +  \y^T \g^{(1)}\|\x_{2:dn}\|_2 + \lp \x_{2:dn}  \rp^T\h^{(1)}\|\y\|_2  -\y^T  \z
\nonumber \\
 \bar{f}_{rd}(c,x;\cG) \hspace{-.12in} & \triangleq &  \hspace{-.12in}
  \min_{\|\x_{2:dn}\|_2=r,\z} \max_{\|\y\|_2=r_y,r_y>0}  \phi(\x,\z,\y)
  \nonumber \\
  \bar{\phi}_0 & \triangleq & \max_{r_y>0}\lim_{n\rightarrow\infty} \frac{1}{dn}
 \lp
 \frac{c_3}{2} r^2r_y^2 -
\frac{1}{c_3} \log \lp \mE_{\cG_{(2)}} e^{ - c_3 \bar{f}_{rd}(c,x;\cG) } \rp   \rp .\label{eq:plta16}
\vspace{-.0in}\end{eqnarray}
One then has \vspace{-.02in}
\begin{eqnarray}
  \lim_{n\rightarrow\infty}\mP_{ A } (f_{rp} (c,x; A )   > \bar{\phi}_0)\longrightarrow 1,\label{eq:plta17a0}
\end{eqnarray}
and
\begin{eqnarray}
\hspace{-.0in}(\bar{\phi}_0  > 0)    \Longrightarrow \lp \lim_{n\rightarrow\infty}\mP_{ A } (f_{rp} (c,x; A )   >0)\longrightarrow 1 \rp
 \Longrightarrow \lp \lim_{n\rightarrow\infty}\mP_{A} \lp \mbox{RdM PR is (uniquely) solvable} \rp \longrightarrow 1\rp.\label{eq:plta17}
\end{eqnarray}
 \label{thm:thm2}
\end{theorem}\vspace{-.17in}
\begin{proof}
Follows as a trivial extension of Theorem 2 in \cite{Stojnicphretreal24}, which for any  fixed $r_y$ is an automatic application of Corollary 3 from  \cite{Stojnicgscompyx16} (see in particular Section 3.2.1 and equation (86) as well as  Lemma 2 and equation (57) in \cite{StojnicMoreSophHopBnds10}). In particular, the lower-bounding side of equation (86) in \cite{Stojnicgscompyx16} corresponds to terms  $ \y^T \g^{(1)}\|\x_{2:n}\|_2$,  $\x_{2:n}^T\h^{(1)}\|\y\|_2$,  $\y^T  \z  $, and $\frac{c_3}{2}r^2r_y^2$. On the other hand, the left hand side of \cite{Stojnicgscompyx16}'s (86) corresponds to $f_{rp}$. An additional maximization over $r_y$ together with concentrations completes the proof.
\end{proof}

Utilization of the results from previous sections as well as those from \cite{Stojnicphretreal24} allows to significantly speed up handling of the
above partially lifted random dual. One first notes that, analogously to (\ref{eq:hrd1}) and (\ref{eq:hrd2}),
\begin{eqnarray}
\bar{f}_{rd}
\hspace{-.05in} & = &
   \min_{\|\x_{2:dn}\|_2=r,\z} \max_{\|\y\|_2=r_y,r_y>0}  \phi(\x,\z,\y)
\nonumber \\
& = &
 \min_{\z}
 \lp
        \sum_{i=1}^{m} \lp \sqrt{\sum_{j=0}^{d-1} \lp \g_{jm+i}^{(0)}\rp^2}
   -
   \sqrt{\sum_{j=0}^{d-1} \z_{jm+i}^2  }
   \rp^2
  +\|\g^{(0)}x   +  \g^{(1)}r -\z\|_2 r_y - \|\h^{(1)} \|_2 r r_y \rp
  \nonumber \\
 & = &
 \min_{\z}
 \lp
        \sum_{i=1}^{m} \lp \| \g_{(i)}^{(0)}\|_2
   -
   \|\z_{(i)}\|_2
   \rp^2
  +\|\g^{(0)}x   +  \g^{(1)}r -\z\|_2 r_y - \|\h^{(1)} \|_2 r r_y \rp
  \nonumber \\
 & = &
 \min_{\z,\gamma>0} \max_{\gamma_{sph}>0}
 \lp
        \sum_{i=1}^{m} \lp \| \g_{(i)}^{(0)}\|_2
   -
   \|\z_{(i)}\|_2
   \rp^2
 + \gamma +  \frac{\|\g^{(0)}x   +  \g^{(1)}r -\z\|_2^2 r_y^2}{4\gamma} -\gamma_{sph} -  \frac{\|\h^{(1)} \|_2^2 r^2 r_y^2}{4\gamma_{sph}}   \rp
  \nonumber \\
 & = &
 \min_{\z,\gamma>0} \max_{\gamma_{sph}>0}
 \lp
        \sum_{i=1}^{m}
        \lp
        \lp \| \g_{(i)}^{(0)}\|_2
   -
   \|\z_{(i)}\|_2
   \rp^2
+ \gamma +  \frac{\|\g_{(i)}^{(0)}x   +  \g_{(i)}^{(1)}r -\z_{(i)}\|_2^2 r_y^2}{4\gamma}
 \rp
 -\gamma_{sph} -  \frac{\|\h^{(1)} \|_2^2 r^2 r_y^2}{4\gamma_{sph}}   \rp,\nonumber \\
\label{eq:plhrd1}
\end{eqnarray}
where $\bar{f}_{rd}$'s arguments are dropped to make writing easier. The above is structurally identical to (34) in \cite{Stojnicphretreal24}. One can then repeat the analysis between \cite{Stojnicphretreal24}'s (34)-(40) to arrive at the following analogue of \cite{Stojnicphretreal24}'s (40)
 \begin{eqnarray}
  \bar{\phi}_0
& \geq &
\max_{r_y>0}\min_{\gamma>0} \max_{\gamma_{sph}>0}
\lp
 \frac{c_3}{2} r^2r_y^2 + \gamma
 -\frac{\alpha}{c_3} \log \lp \mE_{\cG} e^{ - c_3 \bar{f}_{q}  } \rp
 - \gamma_{sph}  - \frac{1}{c_3} \log \lp \mE_{\cG_{(2)}} e^{ c_3 \frac{\lp \h_i^{(1)}\rp^2r^2r_y^2}{4\gamma_{sph}} }\rp
\rp,
\label{eq:plhrd8}
 \end{eqnarray}
where
 \begin{eqnarray}
 \bar{r}_y & = & \frac{r_y^2}{4\gamma} \nonumber \\
 \gamma_x & = & \frac{\bar{r}_y}{1+\bar{r}_y}
 \nonumber \\
 \bar{f}_{q} & = &
\gamma_x \lp \|\g_{(i)}^{(0)} \|_2-   \|\g_{(i)}^{(0)}x   +  \g_{(i)}^{(1)}r \|_2  \rp^2.
 \label{eq:plhrd9}
 \end{eqnarray}
 Following discussion from the previous section we have
 \begin{eqnarray}
\mE_{\cG} e^{ - c_3 \bar{f}_{q}  } & = &
\mE_{\cG} e^{ - c_3 \lp   \gamma_x \lp \|\g_{(i)}^{(0)} \|_2-   \|\g_{(i)}^{(0)}x   +  \g_{(i)}^{(1)}r \|_2  \rp^2  \rp }
\nonumber \\
& = &
\mE_{\cG} e^{ - c_3 \lp   \gamma_x \lp \|\g_{(i)}^{(0)} \|_2-  r \|\frac{1}{r}\g_{(i)}^{(0,u)}x   +  \g_{(i)}^{(1)} \|_2  \rp^2  \rp }
\nonumber \\
& = &
\mE_{u,y} e^{ - c_3   \gamma_x \lp u-  r y  \rp^2  },
 \label{eq:plhrd9a0}
 \end{eqnarray}
 $u$ and $y$ are chi and non-central chi random variables with $d$ degrees of freedom and the following pdfs
 \begin{eqnarray}
 p(u) & = &  \frac{u^{d-1} e^{-\frac{x^2}{2}}}{2^{\frac{d}{2}-1 }  \Gamma\lp\frac{d}{2}\rp }, u\geq 0
 \nonumber \\
  p_1(y,d;\lambda_d) & = &  \frac{ e^{-\frac{y^2+\lambda_d^2}{2}} y^d\lambda_d  } { (\lambda_d y)^{\frac{d}{2}} }
   \cI_{\frac{d}{2}-1} \lp \lambda_d y  \rp, y\geq 0.
  \label{eq:plhrd9a1}
 \end{eqnarray}
One then has
\begin{eqnarray}
f_{q}^{(lift)}   =   \mE_{\cG} e^{-c_3 \bar{f}_{q}}
  =
\mE_{u,y} e^{ - c_3   \gamma_x \lp u-  r y  \rp^2  }
=
\int_{-\infty}^{\infty}
e^{ - c_3   \gamma_x \lp u-  r y  \rp^2  }
\frac{ e^{-\frac{y^2+\lambda_d^2}{2}} y^d\lambda_d  } { (\lambda_d y)^{\frac{d}{2}} }
   \cI_{\frac{d}{2}-1} \lp \lambda_d y  \rp
   \frac{u^{d-1} e^{-\frac{x^2}{2}}}{2^{\frac{d}{2}-1 }  \Gamma\lp\frac{d}{2}\rp }
   dy
   du.
  \label{eq:plhrd11}
 \end{eqnarray}
After setting
 \begin{eqnarray}
\hat{\gamma}_{sph} = \frac{c_3rr_y+\sqrt{c_3^2r^2r_y^2+4}}{4}.
  \label{eq:plhrd13}
 \end{eqnarray}
one has the following analogue to \cite{Stojnicphretreal24}'s (47)
\begin{eqnarray}
 \bar{\phi}_0
& \geq &
 \max_{c_3> 0}  \max_{r_y>0}\min_{\gamma>0}
\Bigg .\Bigg(
 \frac{c_3}{2} r^2r_y^2 + \gamma
 -\frac{\alpha}{c_3} \log \lp f_{q}^{(lift)}\rp
 - \hat{\gamma}_{sph}  +\frac{1}{2c_3} \log \lp  1  -  \frac{c_3rr_y  }{2\hat{\gamma}_{sph}}     \rp
\Bigg.\Bigg).
\label{eq:plhrd16a0}
 \end{eqnarray}

For the completeness, we also observe that inequality signs in (\ref{eq:plhrd8}) and (\ref{eq:plhrd16a0}) can be replaced with equalities.  In Figure \ref{fig:fig3} we present the effect that the above lifting mechanism actually has. To be in agreement with the earlier numerical considerations, we choose again $d=2$ (as stated earlier, this emulates the so-called complex phase retrieval). We selected $\alpha=2.5$ which is significantly lower than $\alpha\approx 2.79$ predicted as the dPR phase transition by the plain RDT. Also, we in parallel show the corresponding plain RDT curve. As can be seen from the figure, there is a fairly strong lifting effect. Moreover, the lifted curve is a decreasing function of $x$ and has a single minimum for $x=1$. Also, its so-called ``\emph{flat region}'' is rather small and moderately good initializers (say the spectral ones from \cite{MondelliM19,LuoAL19,Stojnicphretinit24}) should fall outside this region and allow dPR to work (not only theoretically but also practically) fairly well. We again choose $c=1$ scenario as it is algorithmically the most critical. The lifting effect exists also for $c<1$ but it fairly quickly becomes rather negligible (similar trend happens for $c>1$). Moreover, generic conclusions made earlier when discussing  plain RDT remain valid here as well. The most important of all, as $c$ decreases the curves are even less flat and continue to have a single minimum thereby reassuring that dPRs indeed converge to the global minimum and indeed solve RdM PR.

\begin{figure}[h]
\centering
\centerline{\includegraphics[width=1\linewidth]{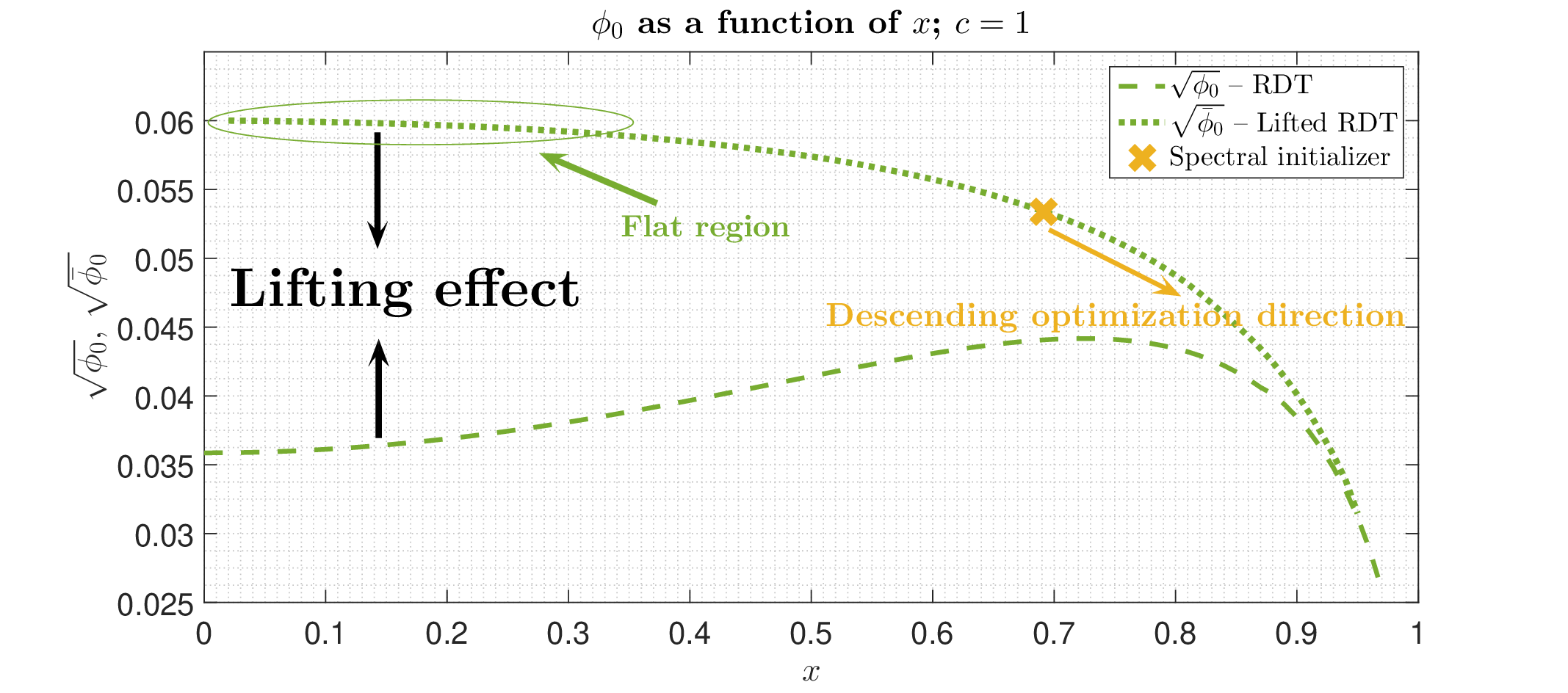}}
\caption{$\phi_0$ as a function of $x$; fixed $\alpha=2.5$; \emph{Lifted} RDT}
\label{fig:fig3}
\end{figure}

To get the limiting phase transition (valid for any nonzero overlap initializer) one would need to check for an $\alpha$ for which the lifted curve remains the flattest in the low $x$ range. Since the flat regions become rather large and the associated sample complexity ratios often unusable in practice there is really not much point in doing so. Instead we selected $\alpha=2.5$ as a value that is not only sufficient to have a decreasing lifted curve but also one with a fairly small flat region that can easily be circumvented with good initializers (for a discussion regarding the role of flat regions and good initializers in successful practical running of dPRs see \cite{Stojnicphretinit24}).

To emphasize another relevant point related to the role of initializers, we in Figures \ref{fig:fig4} and \ref{fig:fig5} selected $\alpha=2.3$ and $\alpha=2.2$, respectively. In both cases the lifted curves are not of the decreasing type. In fact, they are not monotonic either and have ``\emph{two down-sides}'' -- the desired one to the right and the undesired (wrong) one to the left. Depending on which of these sides the initializers' overlap falls descending algorithms may or may not solve the RdM PR. Instead of associating the dPR phase transition with $\alpha$ that results in the decreasing lifted curve, one can consider $\alpha$ such that the lifted curve has a shape that allows initializer to fall on the right down-side (one then also has to ensure that an algorithm which maintains $c=1$ is run or that for $c<1$ initializer remains on the right down-side of the whole parametric manifold). Differently from the above discussed ``\emph{initializer insensitive}'' dPR phase transition, this type of phase transition would clearly depend on the choice of the initializer. If one, for example, relies on the so-called optimal diagonal spectral initializers \cite{MondelliM19,LuoAL19,Stojnicphretinit24} then they fall exactly on the locations indicated in  Figures \ref{fig:fig4} and \ref{fig:fig5}. As Figure \ref{fig:fig4} suggests, one would then have a phase transition around $2.3$.

\begin{figure}[h]
\centering
\centerline{\includegraphics[width=1\linewidth]{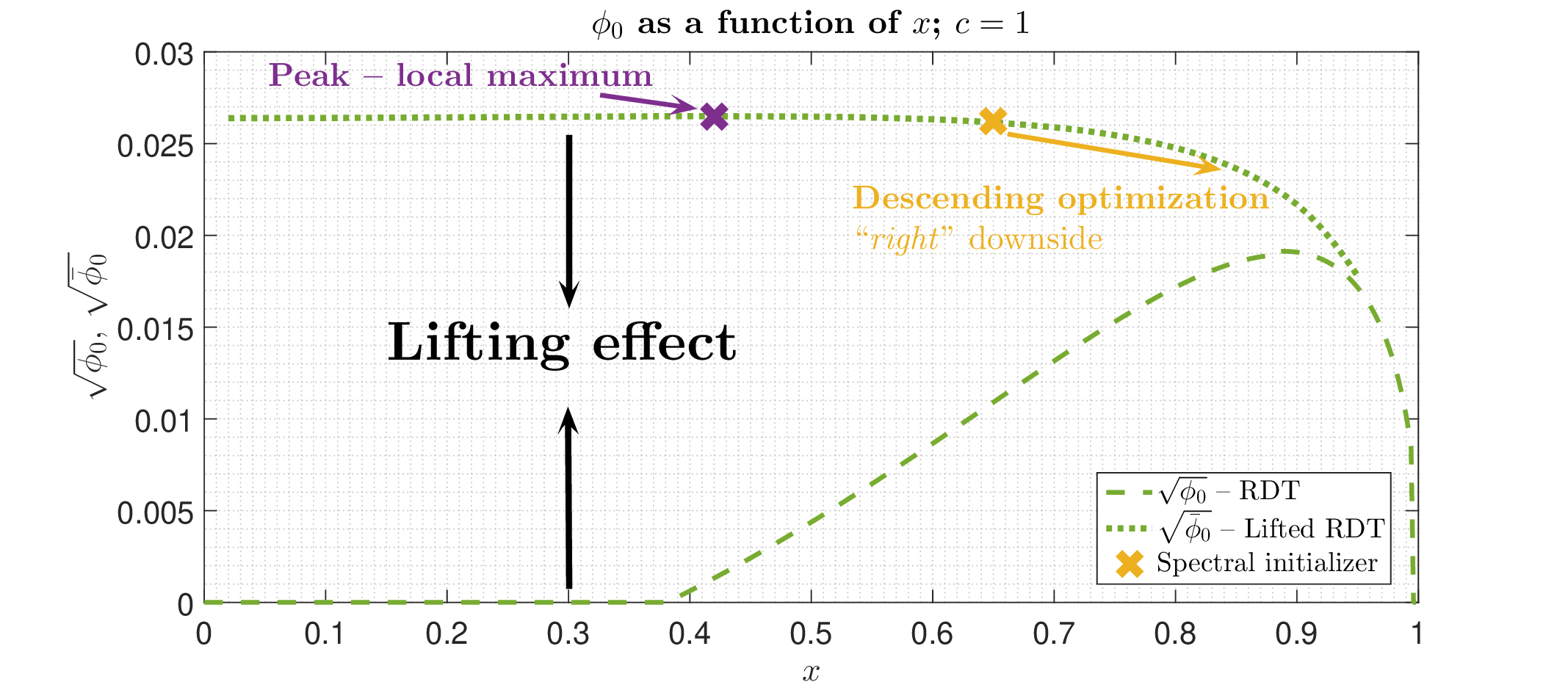}}
\caption{$\phi_0$ as a function of $x$; fixed $\alpha=2.3$; \emph{Lifted} RDT}
\label{fig:fig4}
\end{figure}

\begin{figure}[h]
\centering
\centerline{\includegraphics[width=1\linewidth]{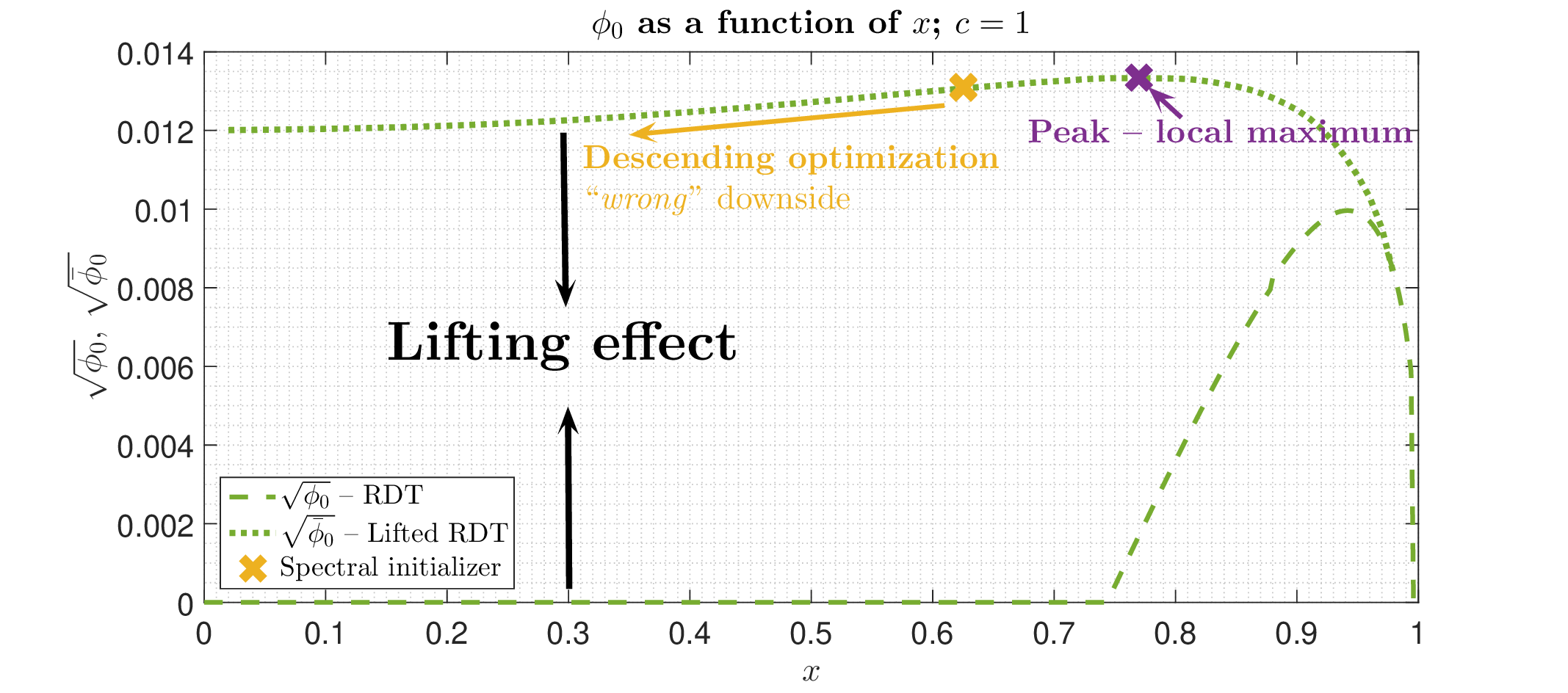}}
\caption{$\phi_0$ as a function of $x$; fixed $\alpha=2.2$; \emph{Lifted} RDT}
\label{fig:fig5}
\end{figure}

The earlier plain  RDT discussion related to local jitteriness also applies to any of the phase transition concepts discussed here. However, as stated above, we intentionally chose $\alpha=2.5$ which is comfortably above the theoretically minimal lifted RDT estimate for (initializer insensitive) dPR phase transition. In other words, when presenting results in Figure \ref{fig:fig3}, we had already kept in mind the earlier mentioned ``\emph{safer compression}'' rule of thumb regarding the choice of sample complexity ratio  being slightly above the theoretical limit. Practical implementations where this can be realistically seen are discussed in next section.

We should also recall on the observations from \cite{Stojnicphretreal24} regarding the role that the  intrinsic random structures properties (beyond optimal objectives landscape) might play in generic PR solvability. In particular, the entire discussion regarding \emph{overlap gap property} (OGP) \cite{Gamar21,GamarSud14,GamarSud17,GamarSud17a,AchlioptasCR11,HMMZ08,MMZ05,Montanari19} and \emph{local entropy} (LE) \cite{Bald15,Bald16,Bald20}  applies here. As discussed in \cite{Stojnicphretreal24}, how relevant all these concepts will turn out to be remains to be seen.

\section{Practical implementations}
 \label{sec:pract}

The above presented theoretical considerations assume that the norm of the optimizing variable, $\sqrt{c}$, remains below $1$. This implies a constrained optimization which makes direct use of the (unconstrained) gradient method (similar to Wirtinger flow) practically inconvenient. To accommodate for this inconvenience, we implemented a log barrier version of the gradient with  the following objective
\begin{eqnarray}
f_{bar}(t_0;\x) \triangleq  t_0\||A\bar{\x}|^2-|A\x|^2 \|_2^2 + \log\lp 1-\|\x\|_2^2 \rp
=t_0f_{plain}(\x) + \log\lp 1-\|\x\|_2^2 \rp,
\label{eq:practeq1}
\end{eqnarray}
where
\begin{eqnarray}
f_{plain}(\x) \triangleq \||A\bar{\x}|^2-|A\x|^2 \|_2^2.
\label{eq:practeq1a0}
\end{eqnarray}
We apply  a gradient with a backtracking optimization procedure, $\mathbf{gradback}$,  to ensure norm constraints are satisfied -- on $f_{bar}(t_0;\x^{(gb,0)}) $ (where $\x^{(gb,0)}$ is a starting point)  and denote output as $\x^{(gb)}$. We then iteratively repeat it for an increasing schedule of $t_0$ (until $t_0$ is sufficiently large, say $10^7$)
\begin{eqnarray}
\bl{\mathbf{gradbar:}} \qquad   \x^{(i+1)}&  =  & \mathbf{gradback}(f_{bar}(t_0^{(i)};\x^{(i)})) \qquad \mbox{and} \qquad  t_0^{(i+1)}=1.6t_0^{(i)}, i=0,1,2,\dots.
\label{eq:practeq2}
\end{eqnarray}
We select  $t_0^{(0)}=0.01$, and typically  rely on a spectral initialization for $\x^{(0))}$, i.e. we take (for $d=2$)
\begin{eqnarray}
\bl{\mbox{\emph{\textbf{diagonal spectral initializer:}}}} \qquad   \x^{(0)}=\x^{(spec)} \triangleq \mbox{max eigenvector} \lp A^T \mbox{diag} \lp  1 - \frac{d}{\bar{\y}^{(d)}}  \rp   A\rp,
\label{eq:practeq6}
\end{eqnarray}
where $ \bar{\y}^{(d)}$ is $\bar{\y}$ replicated $d$ times
\begin{eqnarray}
 \bar{\y}^{(d)} = \begin{bmatrix}
                   \bar{\y}^T & \bar{\y}^T & \dots & \bar{\y}^T
                  \end{bmatrix}^T.
  \label{eq:practeq6a0}
\end{eqnarray}

The results that we  obtained through numerical running of $\mathbf{gradbar}$  are shown in Figure \ref{fig:fig6}.  We chose $n=100$ and denote by $\hat{\x}$ the obtained estimate of $\bar{\x}$. As can be seen, the simulated phase transition is fairly close to the ``safer compression'' adjustment of the theoretical predictions. This happens despite the fact that the dimensions are very small ($n=100$) and that we used squared magnitudes. On the other hand, this is also to be expected and is in an excellent agreement with the discussions in Sections 4.2 and 4.3 in \cite{Stojnicphretreal24}.

\begin{figure}[h]
\centering
\centerline{\includegraphics[width=1\linewidth]{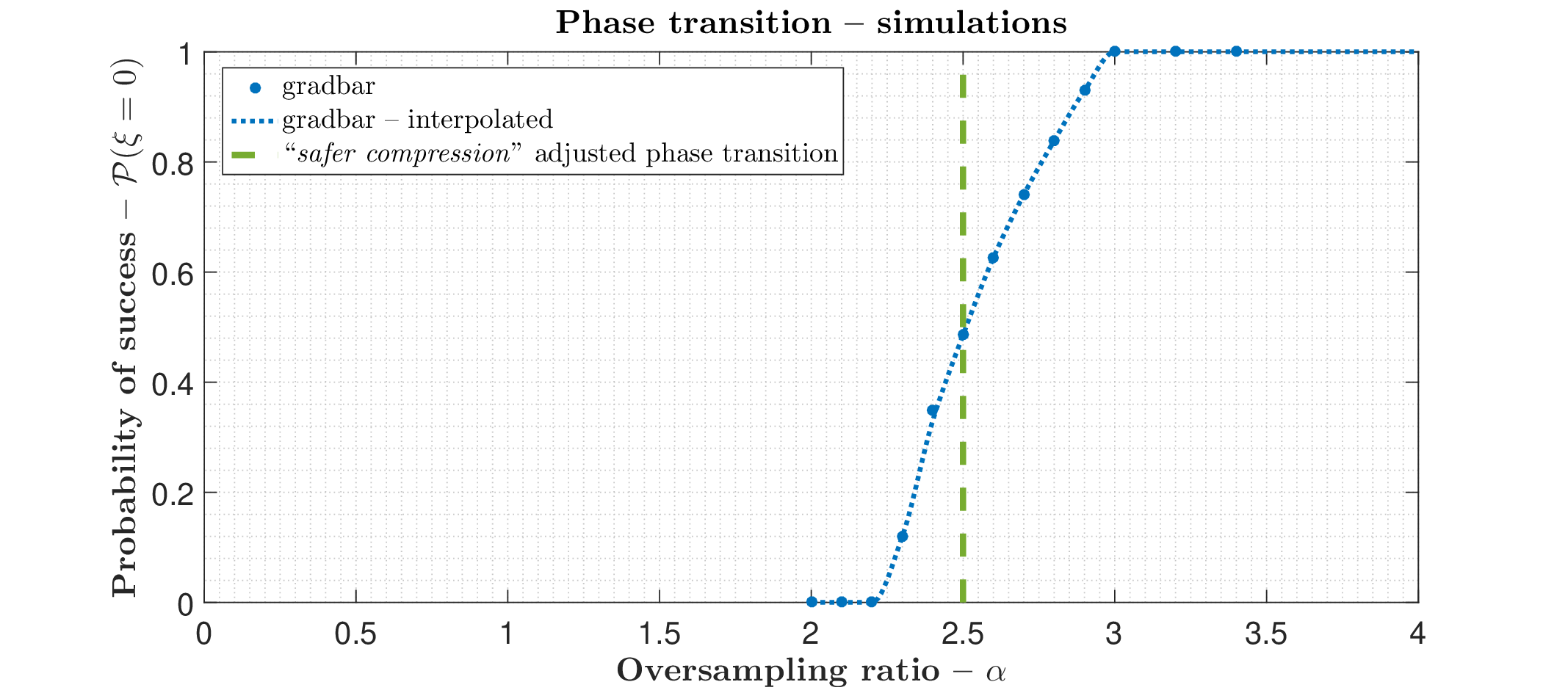}}
\caption{Simulation -- $\mathbf{gradbar}$; $d=2$; $n=100$}
\label{fig:fig6}
\end{figure}

The above $f_{plain}(\x)$ utilizes squared magnitudes in the parenthesis and therefore slightly differs from the corresponding one used in theoretical considerations. The variant from (\ref{eq:practeq1a0}) is practically more convenient as the derivatives are ``\emph{smoother}''.
Its theoretical  analysis would  conceptually follow line by line the one we presented earlier for the non-squared magnitudes. It is just that some of the elegant closed form expressions would not be available which would then additionally burden numerical evaluations and precision of the obtained results. We skip redoing such an exercise and refer to Sections 4.2 and 4.3 in \cite{Stojnicphretreal24} for the obtained results for $d=1$ case are presented. One needs to keep in mind that the theoretical analysis assumes that the initializers overlap in measurement independent. This is not quite the case with the spectral initializers used in the above practical running of $\mathbf{gradbar}$. However, as the obtained results show, this effect seems practically rather unnoticeable. Also, the norm is taken to be smaller than 1 which may be viewed as existence of a prior knowledge of the norm. However, such viewing is not necessary. For example, one can rerun $\mathbf{gradbar}$  $\sim 1/\epsilon$ times (without effectively changing the complexity order) with $\sim 1/\epsilon$ different norms and ultimately obtain results virtually identical (basically in the $\epsilon$ vicinity for any constant $\epsilon>0$) to the ones presented in Figure \ref{fig:fig6}.

\section{Conclusion}
\label{sec:conc}

We study rank $d$ positive definite measurements phase retrieval (RdM PR) generalization and the associated \emph{descending}  algorithms (dPR). \cite{Stojnicphretreal24} developed \emph{Random duality theory} (RDT) based analytical program to characterize performance of dPRs when measurements are of rank 1. We here show how the program can be extended so that it can handle rank $d$ measurements.  The impact of oversampling ratio, $\alpha$, on dPR's performance is numerically evaluated through the developed extension of the RDT program  and emergence of a phase transition is observed. In particular, as $\alpha$ increases, the optimizing objective transitions from a multi to a single local minimum phase, resulting in an analogous dPR transition from a generically non-converging to a generically converging to a global optimum phase.

An anticipated lack of strong random duality (and consequent, strictly bounding, nature of plain RDT results) motivated implementation of a powerful \emph{Lifted} RDT based mechanism. The ``\emph{lifting effect}''  and the resulting sample complexity ratio phase transition adjustment are precise;y characterized as well. Since the obtained theoretical predictions assume heavy concentrations (which for $n\rightarrow\infty$ indeed happen), they are a bit optimistic for practical considerations. To account for the finite dimensional effects and potential appearance of the objective's local jitteriness,  we introduce the ``\emph{safer compression}'' phase transition adjustment concept. It assumes that in practical considerations one relies on a sample complexity ratios, $\alpha$, slightly higher (say $10-20\%$) than the minimally needed phase transitioning one.

Theoretical results are also complemented with the corresponding numerical experiments as well. For rank 2 measurements case (practically highly relevant one that effectively emulates the complex phase retrieval), we implement a log barrier variant of the gradient descent together with the optimal diagonal spectral initializers from \cite{MondelliM19,LuoAL19,Stojnicphretinit24}. Even for small dimensions (on the order of $n=100$) where strong local jitteriness effects are highly likely, the simulated phase transition fairly close matches the corresponding safer compression adjusted theoretical one.


Developed methodologies are fairly general and allow for  many further extensions. Clearly, studying further impact of fl RDT is (from the theoretical point of view) the next logical step. On the other hand, studying signal and/or measurements further structuring, stability, and robustness appear as topics of great practical importance. Concepts presented here apply to any of these studies. Since the relevant technical details are problem specific, we present them in separate papers.

\begin{singlespace}
\bibliographystyle{plain}
\bibliography{nflgscompyxRefs}
\end{singlespace}


\end{document}